\documentclass[11pt]{article}

\usepackage[margin=1in]{geometry}
\usepackage[numbers,sort&compress]{natbib}

\usepackage[utf8]{inputenc}
\usepackage[T1]{fontenc}
\usepackage{url}
\usepackage{amsmath,amsthm,amssymb,amsfonts}
\usepackage{enumerate}
\usepackage{kotex}
\usepackage{multirow}
\usepackage{hhline}
\usepackage{xcolor}
\usepackage{tikz}
\usetikzlibrary{arrows.meta}
\usepackage{mdframed}
\usepackage[dvipdfmx,hypertexnames=false]{hyperref}
\usepackage{graphicx}
\usepackage{graphics}
\usepackage{algorithm}
\usepackage{algpseudocode}
\usepackage{epsfig,psfrag}
\usepackage{tabularx}
\usepackage{tabulary}
\usepackage{booktabs}
\usepackage{nicefrac}
\usepackage{microtype}
\usepackage[noabbrev,capitalise,nameinlink]{cleveref}

\hypersetup{
  breaklinks=true,
  colorlinks=true,
  linkcolor=blue,
  citecolor=blue,
  urlcolor=blue
}

\DeclareMathOperator{\argmax}{arg\,max}
\DeclareMathOperator{\argmin}{arg\,min}
\DeclareMathOperator{\co}{co}
\DeclareMathOperator{\diag}{diag}

\newcommand{\R}{\mathbb{R}}
\newcommand{\E}{\mathbb{E}}

\newcommand{\calF}{\mathcal{F}}
\newcommand{\calS}{\mathcal{S}}
\newcommand{\calA}{\mathcal{A}}
\newcommand{\calH}{\mathcal{H}}

\newcommand{\jsr}{\rho}
\newcommand{\bPi}{\boldsymbol{\Pi}}
\newcommand{\bT}{\mathbf{T}}
\newcommand{\bA}{\mathbf{A}}

\newtheorem{theorem}{Theorem}
\newtheorem{example}{Example}
\newtheorem{assumption}{Assumption}

\newtheorem{corollary}{Corollary}

\newtheorem{lemma}{Lemma}

\newtheorem{proposition}{Proposition}

\crefname{assumption}{Assumption}{Assumptions}
\Crefname{assumption}{Assumption}{Assumptions}

\newmdenv[
  linewidth=0.6pt,
  linecolor=black,
  skipabove=6pt,
  skipbelow=6pt,
  innertopmargin=6pt,
  innerbottommargin=6pt,
  innerleftmargin=8pt,
  innerrightmargin=8pt
]{restatementbox}

\title{A Switching System Theory of Q-Learning with Linear Function Approximation}

\author{%
Donghwan Lee and Han-Dong Lim\\
Department of Electrical Engineering\\
Korea Advanced Institute of Science and Technology (KAIST)\\
Daejeon 34141, South Korea\\
\texttt{donghwan@kaist.ac.kr}
}

\date{}

\begin{document}

\maketitle

\begin{abstract}
Q-learning is a fundamental algorithmic primitive in reinforcement learning. This paper develops a new framework for analyzing linear Q-learning from a switching linear system (SLS) viewpoint, where linear Q-learning denotes Q-learning with linear function approximation. We derive a stochastic SLS representation of the linear Q-learning error and obtain a finite-time error analysis for linear Q-learning through the joint spectral radius (JSR) of the associated SLS family; the JSR is the exact worst-case exponential rate of the corresponding SLSs. The JSR-based rate is tied to the intrinsic worst-case exponential rate of the SLS representation. Moreover, we provide a JSR-based certificate for convergence of linear Q-learning, which can be less conservative than one-step norm bounds.
\end{abstract}

\section{Introduction}
Q-learning~\cite{watkins1992q} is a foundational algorithm for reinforcement learning in discounted Markov decision processes (MDPs)~\cite{sutton1998reinforcement,puterman2014markov,bertsekas2015dynamic}. For large state-action spaces, tabular Q-learning can be computationally prohibitive because it maintains and updates a separate value for every state-action pair. A standard remedy is to approximate the Q-function within a parameterized function class, thereby replacing the search over a high-dimensional table by a search over a lower-dimensional parameter space~\cite{sutton1998reinforcement,bertsekas2015dynamic}. The simplest and most widely studied choice is linear function approximation (LFA), in which the Q-function is represented as a linear combination of features. The convergence and stability of Q-learning with LFA, also called linear Q-learning, have been investigated extensively~\cite{melo2007convergence,carvalho2020new,zhang2021breaking,chen2023target,che2024target,meyn2024projected,liu2025linear}.

In this paper, we study linear Q-learning through switching linear system (SLS) theory~\cite{liberzon2003switching,lin2009stability}. The stability object is the joint spectral radius (JSR) of the associated switching matrix family~\cite{rota1960note,tsitsiklis1997lyapunov,jungers2009joint}.
This representation gives a switched linear model for the deterministic recursion of linear Q-learning. When the corresponding JSR is less than one, a piecewise quadratic Lyapunov norm~\cite{hushenzhang2010generating} certifies contraction of the mean linear Q-learning map, uniqueness of the projected Bellman fixed point, and exponential convergence. The JSR analysis permits arbitrary switching, whereas the deterministic linear Q-learning realizes only the stochastic-policy modes generated along its trajectory. Therefore, the proposed JSR condition is a robust sufficient certificate for the actual recursion. The proposed Lyapunov certificate is then used for stochastic linear Q-learning under independent and identically distributed (i.i.d.) observations, where the sampled recursion is the switched mean dynamics plus martingale-difference noise. We also apply the switching viewpoint to regularized Q-learning with LFA~\cite{limlee2024regularized}. Regularization shifts each direct switching mode, which leads to a regularized JSR and stability conditions depending on the regularization parameter. The analysis gives a unified way to compare unregularized and regularized linear Q-learning through the switched matrix families induced by the Bellman optimality error via stochastic-policy linearization.

\section{Related Work}\label{sec:related-work}
Classical convergence analyses of Q-learning rely on stochastic approximation, ordinary differential equation (ODE) methods, Bellman contraction, monotonicity, and finite-time error estimates~\cite{tsitsiklis1994asynchronous,szepesvari1998asymptotic,kearnssingh1998finite,borkarmeyn2000ode,evendar2003learning,becksrikant2012error,quwierman2020finite,li2021sample,chen2021lyapunov}. Convergence theory for linear Q-learning has also been developed beyond the tabular setting. Early sufficient conditions for convergence with LFA were given by~\cite{melo2007convergence}. Subsequent work introduced convergent variants and stabilization mechanisms, including two-time-scale variants~\cite{carvalho2020new}, target networks~\cite{zhang2021breaking}, target networks with truncation~\cite{chen2023target}, and target networks with over-parameterization~\cite{che2024target}. Finite-sample nonlinear stochastic approximation has also been used to study reinforcement learning algorithms under Markovian noise~\cite{chen2022finite}. Recent work on projected Bellman equations and linear Q-learning has established existence, stability, and bounded-set convergence results under suitable conditions~\cite{meyn2024projected,liu2025linear}. These results address related stability questions through stochastic approximation, algorithmic modification, target network structure, or projected equation analysis.
Linear approximation and projected Bellman equations have been studied in approximate dynamic programming, projected value iteration, linear Q-learning, and regularized Q-learning~\cite{limlee2025understanding,limlee2024regularized,goyalgrandclement2023firstorder}. These works clarify projection structure, approximation error, algorithmic differences, and the role of regularization. The approach here is complementary: it extracts the switched matrix family of the deterministic linear Q-learning recursion and uses its JSR as the stability certificate.

\section{Preliminaries}\label{sec:preliminaries}

\subsection{Notation}
The set of real numbers is denoted by \(\R\); \(\R^m\) is the \(m\)-dimensional Euclidean space; and \(\R^{m\times r}\) is the set of all \(m\times r\) real matrices. For a matrix \(A\), \(A^\top\) denotes its transpose. The identity matrix is denoted by \(I\). For vectors, \(e_i\) is the \(i\)th standard basis vector, with dimension clear from context, and \(\otimes\) denotes the Kronecker product. For a finite set \(\calS\), \(|\calS|\) denotes its cardinality. We write $\Delta_m:=\{q\in\R^m:q_i\ge0,\ \sum_{i=1}^m q_i=1\}$ for the probability simplex in \(\R^m\). For a finite matrix family \(\calH=\{\bA_1,\ldots,\bA_N\}\), $\co(\calH)
:=
\left\{
\sum_{i=1}^N \lambda_i \bA_i:
\lambda_i\ge0,\ \sum_{i=1}^N\lambda_i=1
\right\}$ denotes its convex hull.

\subsection{Switching Systems}\label{subsec:switching-systems}
A discrete-time switching linear system (SLS) is a dynamical system whose state evolves according to one matrix selected from a prescribed family at each time~\cite{liberzon2003switching,lin2009stability,jungers2009joint}.  For a matrix family \(\calH=\{\bA_1,\ldots,\bA_M\}\subset\R^{m\times m}\), the arbitrary-switching system is
\begin{equation*}
  x_{k+1}=\bA_{\sigma_k}x_k,
  \qquad
  \sigma_k\in\{1,\ldots,M\},
  \qquad k\in\{0,1,\ldots\}.
\end{equation*}
The switching signal \(\{\sigma_k\}_{k\ge0}\) may be deterministic, state dependent, or generated by an external process. The SLS is uniformly exponentially stable under arbitrary switching if there exist constants \(C\ge1\) and \(\eta\in(0,1)\) such that
\[
  \|\bA_{\sigma_{k-1}}\cdots \bA_{\sigma_0}x\|_2
  \le
  C\eta^k\|x\|_2
\]
for every horizon \(k\ge 0\), every initial state \(x\in \R^m\), and every switching sequence. A common Lyapunov function for \(\calH\) is a positive definite function that decreases along every mode in the family. In the analysis below, the Bellman maximum in linear Q-learning induces finite-difference stochastic-policy switching, and the relevant Lyapunov function is built directly from products of the induced mode matrices.

\subsection{Joint Spectral Radius}\label{subsec:jsr}
For a bounded set of matrices \(\calH\subset\R^{m\times m}\), its joint spectral radius (JSR)~\cite{rota1960note,tsitsiklis1997lyapunov,jungers2009joint} is
\[
\jsr(\calH)
:=
\lim_{k\to\infty}
\sup_{\bA_1,\ldots,\bA_k\in\calH}
\|\bA_k\cdots \bA_1\|^{1/k},
\]
where the value is independent of the chosen submultiplicative norm. When \(\calH\) is finite, the supremum for each fixed product length is a maximum over products generated by matrices in \(\calH\). The JSR is the exact worst-case exponential growth rate of the switched products generated by \(\calH\). In this paper, a JSR less than one is used to construct a common piecewise quadratic Lyapunov certificate for arbitrary-switching stability and for contraction of the Bellman-generated nonlinear recursion.

The following piecewise quadratic Lyapunov construction is the finite-family common Lyapunov lemma introduced in~\cite{lee2026lyapunovcertified,hushenzhang2010generating}. It is stated here because it is the Lyapunov certificate used throughout the deterministic, stochastic, and regularized analyses.
\begin{lemma}\label{lem:common-lyapunov-construction}
Let
\[
  \calH=\{\bA_1,\bA_2,\ldots,\bA_M\}\subset\R^{m\times m},
  \qquad
  \rho:=\jsr(\calH),
\]
and fix \(\varepsilon>0\) such that \(\beta_\varepsilon:=\rho+\varepsilon\in(0,1)\). For a sequence of modes \(\sigma=(\sigma_1,\ldots,\sigma_k)\in\{1,\ldots,M\}^k\), write
\[
  \bA_\sigma:=\bA_{\sigma_k}\cdots \bA_{\sigma_1},
\]
with the convention that, for \(k=0\), the empty word gives \(\bA_\sigma=I\).  For each integer \(t\ge0\), define
\begin{equation*}
  V_\varepsilon^t(x)
  :=
  \sum_{k=0}^{t}
  \beta_\varepsilon^{-2k}
  \max_{\sigma\in\{1,\ldots,M\}^k}
  \|\bA_\sigma x\|_2^2,
  \qquad x\in\R^m .
\end{equation*}
Then the following statements hold.
\begin{enumerate}[(i)]
\item For every \(t\ge0\),
\[
  V_\varepsilon^{t+1}(x)
  \ge
  \|x\|_2^2
  +
  \beta_\varepsilon^{-2}
  \max_{i\in\{1,\ldots,M\}}V_\varepsilon^t(\bA_i x),
  \qquad \forall x\in\R^m .
\]
\item For every \(t\ge0\), the function \(V_\varepsilon^t\) is absolutely homogeneous of degree two and is monotone in \(t\):
\[
  V_\varepsilon^t(\lambda x)=|\lambda|^2V_\varepsilon^t(x),
  \qquad
  V_\varepsilon^t(x)\le V_\varepsilon^{t+1}(x).
\]
\item There exists \(C_\varepsilon>0\) such that
\[
  \|x\|_2^2
  \le
  V_\varepsilon^t(x)
  \le
  C_\varepsilon\|x\|_2^2,
  \qquad
  \forall x\in\R^m,
  \quad
  \forall t\ge0 .
\]
\item The pointwise limit
\[
  V_\varepsilon^\infty(x)
  :=
  \lim_{t\to\infty}V_\varepsilon^t(x)
\]
exists and satisfies
\[
  \|x\|_2^2
  \le
  V_\varepsilon^\infty(x)
  \le
  C_\varepsilon\|x\|_2^2,
  \qquad
  \forall x\in\R^m .
\]
\item The function \(p_\varepsilon(x):=\sqrt{V_\varepsilon^\infty(x)}\) is a norm on \(\R^m\).
\item For every \(i\in\{1,\ldots,M\}\),
\begin{equation*}
  V_\varepsilon^\infty(\bA_i x)
  \le
  \beta_\varepsilon^2
  \left(V_\varepsilon^\infty(x)-\|x\|_2^2\right)
  \le
  \beta_\varepsilon^2V_\varepsilon^\infty(x),
  \qquad \forall x\in\R^m .
\end{equation*}
Equivalently, \(p_\varepsilon(\bA_i x)\le\beta_\varepsilon p_\varepsilon(x)\) for every \(x\in\R^m\).
\end{enumerate}
\end{lemma}

\subsection{Discounted Markov Decision Processes and Linear Function Approximation}
In this paper, we consider a finite discounted Markov decision process (MDP)~\cite{puterman2014markov} with state-space \(\calS=\{1,\ldots,|\calS|\}\), action-space \(\calA=\{1,\ldots,|\calA|\}\), transition probability \(P(s'\mid s,a)\), real-valued one-step reward \(r(s,a,s')\), expected reward
\[
R(s,a):=\sum_{s'\in\calS}P(s'\mid s,a)r(s,a,s'),
\]
and discount factor \(\gamma\in(0,1)\). State-action functions are viewed as vectors in \(\R^{|\calS||\calA|}\) using the action-block ordering $(1,1),(2,1),\ldots,(|\calS|,1),(1,2),(2,2),\ldots,(|\calS|,|\calA|)$.
All matrices and vectors indexed by state-action pairs use this ordering. Let us define the matrix
\[
P:=
\begin{bmatrix}
P_1\\ \vdots\\ P_{|\calA|}
\end{bmatrix}
\in\R^{|\calS||\calA|\times |\calS|},
\qquad
R:=
\begin{bmatrix}
R(\cdot,1)\\ \vdots\\ R(\cdot,|\calA|)
\end{bmatrix}
\in\R^{|\calS||\calA|},
\]
where \(P_a=P(\cdot\mid\cdot,a)\in\R^{|\calS|\times|\calS|}\).
For the finite MDP above, let us define
\[
  R_{\max}:=\max_{(s,a,s')\in\calS\times\calA\times\calS}|r(s,a,s')|.
\]
Because the state and action spaces are finite and rewards are real-valued, \(R_{\max}<\infty\).

Let \(\Theta\) denote the set of deterministic stationary policies \(\pi:\calS\to\calA\). For any stochastic policy \(\mu:\calS\to\Delta_{|\calA|}\), we define
\[
\bPi^\mu
:=
\begin{bmatrix}
\mu(1)^\top\otimes e_1^\top\\
\mu(2)^\top\otimes e_2^\top\\
\vdots\\
\mu(|\calS|)^\top\otimes e_{|\calS|}^\top
\end{bmatrix}
\in\R^{|\calS|\times |\calS||\calA|}.
\]
For a deterministic policy \(\pi\in\Theta\), we use the same notation \(\bPi^\pi\) by identifying \(\pi(s)\) with its one-hot encoding. For \(Q\in\R^{|\calS||\calA|}\), define
\[
  V_Q(s):=\max_{a\in\calA}Q(s,a),
  \qquad
  V_Q:=(V_Q(1),\ldots,V_Q(|\calS|))^\top.
\]
The Bellman optimality operator is
\[
  F(Q):=R+\gamma PV_Q.
\]
Let \(\Phi\in\R^{|\calS||\calA|\times m}\) be a feature matrix. In this paper, we assume that the feature matrix \(\Phi\in\R^{|\calS||\calA|\times m}\) has full column rank, which is standard in the literature.
Its row corresponding to \((s,a)\) is denoted by \(\phi(s,a)^\top\), where \(\phi(s,a)\in\R^m\). The linear function approximation (LFA) of the Q-function is
\[
  Q_\theta:=\Phi\theta.
\]
For an LFA parameter \(\theta\), define the corresponding value function with the greedy policy
\[
  V_\theta(s):=\max_{a\in\calA}\phi(s,a)^\top\theta,
  \qquad
  V_\theta:=(V_\theta(1),\ldots,V_\theta(|\calS|))^\top .
\]

We use \(d\) to denote a state-action sampling distribution on \(\calS\times\calA\). In the independent and identically distributed (i.i.d.) observation model, \(d\) is the sampling distribution of \((s_k,a_k)\).
Throughout the paper, we assume that the sampling distribution satisfies \(d(s,a)>0\) for every \((s,a)\in {\mathcal S}\times {\mathcal A}\).
Under the assumption, \(\Phi^\top D\Phi\succ0\). Since the state-action space is finite, the feature radius
\[
  \phi_{\max}:=\max_{(s,a)\in\calS\times\calA}\|\phi(s,a)\|_2
\]
is finite as well.

\subsection{Projected Bellman Equation and Projected Q-Value Iteration}
For a sampling distribution \(d\) on state-action pairs, define the diagonal weighting matrix
\[
  D:=\diag(d(s,a))_{(s,a)\in\calS\times\calA}
  \in \R^{|\calS||\calA|\times |\calS||\calA|}.
\]
Whenever \(\Phi^\top D\Phi\) is nonsingular, the \(D\)-orthogonal projection onto \(\operatorname{range}(\Phi)\) is
\[
  \bPi_D
  :=
  \Phi(\Phi^\top D\Phi)^{-1}\Phi^\top D .
\]
The projected Bellman equation~\cite{limlee2025understanding} is
\begin{equation}\label{eq:pbe-q-space-general}
  \Phi\theta
  =
  \bPi_D F(\Phi\theta)
  =
  \bPi_D\left(R+\gamma PV_\theta\right).
\end{equation}
A projected Bellman fixed point is a parameter \(\theta^\star\) satisfying
\begin{equation*}
  \Phi\theta^\star
  =
  \bPi_D F(\Phi\theta^\star)
  =
  \bPi_D\left(R+\gamma PV_{\theta^\star}\right).
\end{equation*}
In general, such a fixed point need not exist, and when it exists it need not be unique~\cite{limlee2025understanding}.
The projected Q-value iteration (PQVI) associated with \cref{eq:pbe-q-space-general} is
\[
  \Phi\theta_{k+1}^{\mathrm{PQVI}}
  =
  \bPi_D\left(R+\gamma PV_{\theta_k^{\mathrm{PQVI}}}\right),
  \qquad k\in\{0,1,\ldots\},
\]
or, equivalently, in parameter space,
\[
  \theta_{k+1}^{\mathrm{PQVI}}
  =
  (\Phi^\top D\Phi)^{-1}\Phi^\top D
  \left(R+\gamma PV_{\theta_k^{\mathrm{PQVI}}}\right),
  \qquad k\in\{0,1,\ldots\}.
\]

\subsection{Linear Q-Learning and Deterministic Linear Q-Learning}
Given a transition sample \((s_k,a_k,r_{k+1},s'_k)\), where \(r_{k+1}=r(s_k,a_k,s'_k)\), the scalar step-size linear Q-learning update is
\[
\theta_{k+1}
=
\theta_k+
\alpha\phi(s_k,a_k)
\left(
 r_{k+1}
 +
 \gamma\max_{u\in\calA}\phi(s'_k,u)^\top\theta_k
 -
 \phi(s_k,a_k)^\top\theta_k
\right),\qquad k\in\{0,1,\ldots\},
\]
where $\alpha \in (0,1)$ is the step-size, and the initial parameter \(\theta_0\) is assumed to be deterministic throughout this paper.
The parameter \(\theta_k\) determines the approximate action-value function \(Q_{\theta_k}=\Phi\theta_k\). The term inside parentheses is the temporal-difference error formed from the greedy next-action value under the same parameter. The rest of the paper studies the deterministic averaged version of this update and then returns to sampled updates under the i.i.d.\ observation model.

Taking expectation with respect to the sampling distribution and the transition kernel gives the deterministic linear Q-learning recursion
\begin{equation}\label{eq:deterministic-lfa-recursion}
  \theta_{k+1}
  =
  \theta_k+
  \alpha\Phi^\top D
  \left(R+\gamma PV_{\theta_k}-\Phi\theta_k\right),\qquad k\in\{0,1,\ldots\}.
\end{equation}
Let us define the deterministic linear Q-learning map
\begin{equation}\label{eq:Talpha-def}
  \bT_\alpha(\theta)
  :=
  \theta+
  \alpha g(\theta),
\end{equation}
where
\begin{equation}\label{eq:g-theta-def}
  g(\theta)
  :=
  \Phi^\top D\left(R+\gamma PV_\theta-\Phi\theta\right)
\end{equation}
will be called the projected Bellman residual. With this notation, \cref{eq:deterministic-lfa-recursion} is written compactly as
\[
  \theta_{k+1}=\bT_\alpha(\theta_k),
  \qquad k\in\{0,1,\ldots\}.
\]

The following lemma states the exact relation between the projected Bellman equation in \cref{eq:pbe-q-space-general} and the projected Bellman residual equation \(g(\theta)=0\).
\begin{lemma}\label{lem:pbe-normal-equivalence}
Suppose that \(\Phi^\top D\Phi\) is nonsingular. Then the solution set of \(g(\theta)=0\) is identical to the solution set of the projected Bellman equation~\cref{eq:pbe-q-space-general}.
\end{lemma}

\begin{proof}
Since \(\Phi^\top D\Phi\) is nonsingular by assumption, the projected Bellman equation \cref{eq:pbe-q-space-general} can be written as
\begin{align}
  \Phi\theta
  =
  \Phi(\Phi^\top D\Phi)^{-1}\Phi^\top D
  \left(R+\gamma PV_\theta\right).\label{eq:1}
\end{align}
We prove the two implications separately.

First, suppose that \(\theta\) satisfies the projected Bellman equation. Multiplying~\cref{eq:1}
from the left by \(\Phi^\top D\) gives
\[
  \Phi^\top D\Phi\theta
  =
  \Phi^\top D\Phi(\Phi^\top D\Phi)^{-1}\Phi^\top D
  \left(R+\gamma PV_\theta\right)
  =
  \Phi^\top D\left(R+\gamma PV_\theta\right).
\]
Therefore
\[
  g(\theta)
  =
  \Phi^\top D\left(R+\gamma PV_\theta-\Phi\theta\right)
  =0.
\]
Thus every solution of the projected Bellman equation is a zero of the normal-equation residual.

Conversely, suppose that \(g(\theta)=0\).  Then
\[
  0=
  \Phi^\top D\left(R+\gamma PV_\theta-\Phi\theta\right)
  =
  \Phi^\top D\left(R+\gamma PV_\theta\right)-\Phi^\top D\Phi\theta,
\]
and hence
\[
  \Phi^\top D\Phi\theta
  =
  \Phi^\top D\left(R+\gamma PV_\theta\right).
\]
Since \(\Phi^\top D\Phi\) is nonsingular,
\[
  \theta
  =
  (\Phi^\top D\Phi)^{-1}\Phi^\top D
  \left(R+\gamma PV_\theta\right).
\]
Multiplying this identity by \(\Phi\) yields
\[
  \Phi\theta
  =
  \Phi(\Phi^\top D\Phi)^{-1}\Phi^\top D
  \left(R+\gamma PV_\theta\right)
  =
  \bPi_D\left(R+\gamma PV_\theta\right).
\]
This is exactly \cref{eq:pbe-q-space-general}. Therefore, every zero of \(g\) solves the projected Bellman equation.
\end{proof}

By~\cref{lem:pbe-normal-equivalence}, a projected Bellman fixed point \(\theta^\star\) can equivalently be characterized by the equation
\begin{equation}\label{eq:pbe}
\Phi^\top D\left(R+\gamma PV_{\theta^\star}-\Phi\theta^\star\right) = g(\theta^\star)=0.
\end{equation}
Therefore, a projected Bellman fixed point is precisely a zero of the projected Bellman residual, \(g(\theta^\star)=0\). Equivalently, because \(\bT_\alpha(\theta)=\theta+\alpha g(\theta)\) and \(\alpha>0\), the same point satisfies
\[
  \bT_\alpha(\theta^\star)=\theta^\star .
\]

It is important to distinguish the deterministic linear Q-learning recursion in~\cref{eq:deterministic-lfa-recursion} from projected Q-VI\@. The deterministic linear Q-learning recursion can be written as a residual step toward the projected Q-VI update:
\[
  \theta_{k+1}
  =
  \theta_k+
  \alpha\Phi^\top D\Phi
  \left(\theta_{k+1}^{\mathrm{PQVI}}-\theta_k\right),
  \qquad k\in\{0,1,\ldots\}.
\]
The two iterations coincide only in the special scalar step-size case
\[
  \alpha\Phi^\top D\Phi=I.
\]
Equality of the fixed points of the deterministic linear Q-learning and the projected Q-VI does not imply equality of their convergence behavior. In general, convergence of one iteration does not automatically imply convergence of the other. A closely related discussion appears in~\cite{limlee2025understanding}, where examples show that convergence of either one of the two iterations need not guarantee convergence of the other. The following two examples make this separation explicit.

\begin{example}\label{ex:elq-converges-pqvi-diverges}
Consider a two-state MDP with one action, zero reward, discount factor $\gamma=0.9$, $\alpha=0.1$, and transition matrix
\[
  P=
  \begin{bmatrix}
  0 & 1\\
  0 & 1
  \end{bmatrix}.
\]
Let the sampling distribution and feature matrix be
\[
  d(1,1)=0.99,
  \qquad
  d(2,1)=0.01,
  \qquad
  \Phi=
  \begin{bmatrix}
  1\\ -10
  \end{bmatrix}.
\]
Since there is only one action, the maximization is trivial and \(V_\theta=\Phi\theta\).  Hence
\[
  \Phi^\top D\Phi
  =0.99\cdot 1^2+0.01\cdot (-10)^2
  =1.99,
\]
and, since \(P\Phi=(-10,-10)^\top\),
\[
  \Phi^\top DP\Phi
  =0.99\cdot 1\cdot (-10)+0.01\cdot (-10)\cdot (-10)
  =-8.9.
\]
The projected Q-VI is the scalar recursion
\[
  \theta_{k+1}^{\mathrm{PQVI}}
  =\gamma\frac{\Phi^\top DP\Phi}{\Phi^\top D\Phi}\theta_k
  =-\frac{8.01}{1.99}\theta_k,
  \qquad k\in\{0,1,\ldots\}.
\]
Since \(8.01/1.99>1\), projected Q-VI diverges for every nonzero initial condition.  On the other hand, the deterministic linear Q-learning recursion is
\[
\begin{aligned}
  \theta_{k+1}
  &=\left(1-\alpha\Phi^\top D\Phi+\alpha\gamma\Phi^\top DP\Phi\right)\theta_k \\
  &=\left(1-0.1\cdot1.99+0.1\cdot0.9\cdot(-8.9)\right)\theta_k
  =0,
  \qquad k\in\{0,1,\ldots\}.
\end{aligned}
\]
Thus deterministic linear Q-learning reaches the projected Bellman fixed point \(\theta^\star=0\) in one step, while projected Q-VI diverges.
\end{example}

\begin{example}\label{ex:pqvi-converges-elq-diverges}
Consider a one-state MDP with one action, zero reward, deterministic self-transition, $\gamma=0.9$, $\alpha=0.5$, $d(1,1)=1$, and the one-dimensional feature representation $\Phi=10$. Again, the maximization is trivial. In this case
\[
  \Phi^\top D\Phi=100,
  \qquad
  \Phi^\top DP\Phi=100.
\]
Therefore projected Q-VI satisfies
\[
  \theta_{k+1}^{\mathrm{PQVI}}
  =\gamma\frac{\Phi^\top DP\Phi}{\Phi^\top D\Phi}\theta_k
  =0.9\theta_k,
  \qquad k\in\{0,1,\ldots\},
\]
which converges to \(0\).  However, deterministic linear Q-learning satisfies
\[
\begin{aligned}
  \theta_{k+1}
  &=\left(1-\alpha\Phi^\top D\Phi+\alpha\gamma\Phi^\top DP\Phi\right)\theta_k \\
  &=\left(1-0.5\cdot100+0.5\cdot0.9\cdot100\right)\theta_k
  =-4\theta_k,
  \qquad k\in\{0,1,\ldots\}.
\end{aligned}
\]
Thus deterministic linear Q-learning diverges for every nonzero initial condition, even though projected Q-VI converges to the same fixed point.
\end{example}

\section{Analysis of Deterministic Linear Q-Learning}

\subsection{Switching System Representation}

The next lemma is the key step that expresses Q-learning recursions as switched-system dynamics, and it is motivated by~\cite{goyalgrandclement2023firstorder}.
\begin{lemma}\label{lem:lfa-max-linearization}
For every \(\theta,\bar\theta\in\R^m\), there exists a stochastic policy \(\mu_{\theta,\bar\theta}:\calS\to\Delta_{|\calA|}\) such that
\[
  V_\theta-V_{\bar\theta}
  =
  \bPi^{\mu_{\theta,\bar\theta}}\Phi(\theta-\bar\theta).
\]
Moreover, \(\mu_{\theta,\bar\theta}\) can be chosen as a measurable function of \((\theta,\bar\theta)\).
\end{lemma}

\begin{proof}
Fix \(\theta,\bar\theta\in\R^m\) and a state \(s\in\calS\).  Let
\[
  e(s,a):=(\Phi(\theta-\bar\theta))(s,a).
\]
We show that
\[
  y_s:=V_\theta(s)-V_{\bar\theta}(s)
\]
lies in the interval between \(\min_a e(s,a)\) and \(\max_a e(s,a)\).  Since
\[
V_\theta(s)
=
\max_a\{(\Phi\bar\theta)(s,a)+e(s,a)\},
\]
we have
\[
V_\theta(s)
\le
V_{\bar\theta}(s)+\max_a e(s,a),
\]
and hence \(y_s\le\max_a e(s,a)\).  If \(a_\star\in\argmax_a(\Phi\bar\theta)(s,a)\), then
\[
V_\theta(s)
\ge
(\Phi\bar\theta)(s,a_\star)+e(s,a_\star)
\ge
V_{\bar\theta}(s)+\min_a e(s,a),
\]
so \(y_s\ge\min_a e(s,a)\).  Thus \(y_s\) belongs to the convex hull of the finite set \(\{e(s,a):a\in\calA\}\).

Choose lexicographic minimizers and maximizers
\[
  a_{\min}(s)\in\argmin_a e(s,a),
  \qquad
  a_{\max}(s)\in\argmax_a e(s,a).
\]
If the minimum and maximum are equal, assign probability one to \(a_{\min}(s)\).  Otherwise define
\[
  \lambda_s
  :=
  \frac{y_s-e(s,a_{\min}(s))}
  {e(s,a_{\max}(s))-e(s,a_{\min}(s))}
  \in[0,1]
\]
and set
\[
  \mu_{\theta,\bar\theta}(a_{\max}(s)\mid s)=\lambda_s,
  \qquad
  \mu_{\theta,\bar\theta}(a_{\min}(s)\mid s)=1-\lambda_s,
\]
with all other action probabilities set to zero.  Then
\[
  y_s
  =
  \sum_{a\in\calA}\mu_{\theta,\bar\theta}(a\mid s)e(s,a).
\]
Stacking these identities over all states gives
\[
  V_\theta-V_{\bar\theta}
  =
  \bPi^{\mu_{\theta,\bar\theta}}\Phi(\theta-\bar\theta).
\]
The lexicographic rule makes the construction single-valued.  Because it uses finitely many comparisons of continuous functions and explicit continuous formulas away from equality regions, the selected policy is measurable.
\end{proof}

The above lemma converts a nonlinear maximization difference into a policy-indexed linear operator.
For a stochastic policy \(\mu\), let us define
\[
\bA_\mu
:=
I-
\alpha\Phi^\top D\Phi
+
\alpha\gamma\Phi^\top DP\bPi^\mu\Phi
\in\R^{m\times m}.
\]
The next proposition shows that the deterministic linear Q-learning update inherits identical linear switching system representation.
\begin{proposition}\label{prop:pairwise-direct-lfa}
For every \(\theta,\bar\theta\in\R^m\), there exists a stochastic policy \(\mu_{\theta,\bar\theta}\) such that
\[
  \bT_\alpha(\theta)-\bT_\alpha(\bar\theta)
  =
  \bA_{\mu_{\theta,\bar\theta}}(\theta-\bar\theta).
\]
In particular, if \(\theta^\star\) is a projected Bellman fixed point and \(x_k:=\theta_k-\theta^\star\), then the deterministic linear Q-learning recursion \(\theta_{k+1}=\bT_\alpha(\theta_k)\) satisfies
\begin{equation}\label{eq:det-direct-error}
  x_{k+1}=\bA_{\mu_k} x_k,  \qquad k\in\{0,1,\ldots\},
\end{equation}
where \(\mu_k\) is a stochastic policy depending measurably on \(\theta_k\) and \(\theta^\star\).
\end{proposition}

\begin{proof}
Using \cref{eq:Talpha-def},
\[
\begin{aligned}
\bT_\alpha(\theta)-\bT_\alpha(\bar\theta)
&=
\theta-\bar\theta
+\alpha\Phi^\top D
\left(\gamma P(V_\theta-V_{\bar\theta})-\Phi(\theta-\bar\theta)\right).
\end{aligned}
\]
By \cref{lem:lfa-max-linearization}, there exists a stochastic policy \(\mu_{\theta,\bar\theta}\) such that
\[
  V_\theta-V_{\bar\theta}
  =
  \bPi^{\mu_{\theta,\bar\theta}}\Phi(\theta-\bar\theta).
\]
Substitution gives
\[
\begin{aligned}
\bT_\alpha(\theta)-\bT_\alpha(\bar\theta)
&=
\left(
I-
\alpha\Phi^\top D\Phi
+
\alpha\gamma\Phi^\top DP\bPi^{\mu_{\theta,\bar\theta}}\Phi
\right)(\theta-\bar\theta)\\
&=
\bA_{\mu_{\theta,\bar\theta}}(\theta-\bar\theta).
\end{aligned}
\]
If \(\theta^\star\) is a fixed point and \(\bar\theta=\theta^\star\), then \(\bT_\alpha(\theta^\star)=\theta^\star\), and \cref{eq:det-direct-error} follows.
\end{proof}

An explicit trajectory-level example illustrating \cref{prop:pairwise-direct-lfa} is given in \cref{ex:det-qlearning-switching-trajectory}.
Therefore, the deterministic linear Q-learning recursion in~\cref{eq:deterministic-lfa-recursion} is exactly represented by the switched linear error recursion in~\cref{eq:det-direct-error}, relative to a projected Bellman fixed point. This allows the convergence of deterministic linear Q-learning to be analyzed through the stability of~\cref{eq:det-direct-error}.

The set of all possible such matrices corresponding to deterministic policies is defined as
\[
\mathcal A_\alpha
:=
\left\{
\bA_\pi
:=
I-
\alpha\Phi^\top D\Phi
+
\alpha\gamma\Phi^\top DP\bPi^\pi\Phi:
\pi\in\Theta
\right\}.
\]
The JSR of this finite family is \(\jsr(\mathcal A_\alpha)\).

The JSR \(\jsr(\mathcal A_\alpha)\) above is defined for the finite deterministic matrix family \(\mathcal A_\alpha\).  However, the switched system in \cref{eq:det-direct-error} can involve stochastic-policy modes \(\bA_\mu\), so a direct arbitrary-switching analysis would appear to require the JSR over all such stochastic-policy matrices.  To close this gap, \cref{app:direct-convexification} proves in \cref{lem:lfa-convex-hull} that every stochastic-policy mode lies in the convex hull of the finite deterministic modes, and that the JSR of this convex hull is equal to the finite-family JSR\@.  Thus the finite matrices in \(\mathcal A_\alpha\) are sufficient for certifying all Bellman-induced stochastic modes.

\subsection{Analysis of Deterministic Linear Q-Learning}

This section analyzes~\cref{eq:deterministic-lfa-recursion} using the JSR of the switching family \(\mathcal A_\alpha\). The first step is to build a Lyapunov norm directly from products of the finite switching family $\mathcal A_\alpha$.
To this end, let us assume
\[
  \jsr(\mathcal A_\alpha)<1.
\]
Fix any \(\varepsilon>0\) such that
\[
  \beta_\varepsilon
  :=
  \jsr(\mathcal A_\alpha)+\varepsilon
  <1.
\]
For each integer \(t\ge0\), let us define
\[
V_{\varepsilon}^t(x)
:=
\sum_{\ell=0}^{t}
\beta_\varepsilon^{-2\ell}
\max_{\pi_1,\ldots,\pi_\ell\in\Theta}
\left\|
\bA_{\pi_\ell}\cdots \bA_{\pi_1} x
\right\|_2^2,
\qquad x\in\R^m,
\]
where the \(\ell=0\) term is \(\|x\|_2^2\). Now, define
\begin{equation}\label{eq:lfa-Veps-infty}
  V_{\varepsilon}^\infty(x)
  :=
  \lim_{t\to\infty}V_{\varepsilon}^t(x).
\end{equation}
The piecewise quadratic construction below turns the JSR condition into a global contraction certificate for the nonlinear projected Bellman map. The following theorem states the deterministic convergence result.
\begin{theorem}\label{thm:deterministic-jsr-lfa}
Suppose that
\[
  \jsr(\mathcal A_\alpha)<1.
\]
Fix \(\varepsilon>0\) such that \(\beta_\varepsilon:=\jsr(\mathcal A_\alpha)+\varepsilon<1\).  Then \(V_{\varepsilon}^\infty\) in \cref{eq:lfa-Veps-infty} is well-defined, and there exists \(C_{\varepsilon}\ge1\) such that
\begin{equation}\label{eq:lfa-Veps-equivalence}
  \|x\|_2^2
  \le
  V_{\varepsilon}^\infty(x)
  \le
  C_{\varepsilon}\|x\|_2^2,
  \qquad
  \forall x\in\R^m.
\end{equation}
Moreover,
\[
  p_{\varepsilon}(x)
  :=
  \sqrt{V_{\varepsilon}^\infty(x)}
\]
is a norm, and for every stochastic policy \(\mu\),
\begin{equation}\label{eq:lfa-strong-drift}
V_{\varepsilon}^\infty(\bA_\mu x)
\le
\beta_\varepsilon^2
\left(
V_{\varepsilon}^\infty(x)-\|x\|_2^2
\right)
\le
\beta_\varepsilon^2V_{\varepsilon}^\infty(x).
\end{equation}
Consequently,
\begin{equation}\label{eq:lfa-p-contraction-mode}
  p_{\varepsilon}(\bA_\mu x)
  \le
  \beta_\varepsilon p_{\varepsilon}(x),
  \qquad
  \forall x\in\R^m.
\end{equation}
The deterministic map \(\bT_\alpha\) is a global contraction in \(p_{\varepsilon}\):
\begin{equation}\label{eq:lfa-T-contraction}
  p_{\varepsilon}(\bT_\alpha(\theta)-\bT_\alpha(\bar\theta))
  \le
  \beta_\varepsilon p_{\varepsilon}(\theta-\bar\theta),
  \qquad
  \forall \theta,\bar\theta\in\R^m.
\end{equation}
Hence there exists a unique projected Bellman fixed point \(\theta^\star\) satisfying \cref{eq:pbe}.  For the deterministic recursion \cref{eq:deterministic-lfa-recursion},
\begin{equation}\label{eq:lfa-det-V-convergence}
  V_{\varepsilon}^\infty(\theta_k-\theta^\star)
  \le
  \beta_\varepsilon^{2k}
  V_{\varepsilon}^\infty(\theta_0-\theta^\star),
\end{equation}
and therefore
\[
  \|\theta_k-\theta^\star\|_2
  \le
  \sqrt{C_{\varepsilon}}\,
  \beta_\varepsilon^k
  \|\theta_0-\theta^\star\|_2.
\]
In Q-function norm,
\[
  \|\Phi\theta_k-\Phi\theta^\star\|_2
  \le
  \|\Phi\|_2\sqrt{C_{\varepsilon}}\,
  \beta_\varepsilon^k
  \|\theta_0-\theta^\star\|_2.
\]
\end{theorem}

\begin{proof}
Apply \cref{lem:common-lyapunov-construction} to the finite family \(\mathcal A_\alpha\).  This proves the existence of \(V_{\varepsilon}^\infty\), the norm equivalence \cref{eq:lfa-Veps-equivalence}, and the deterministic-policy inequality.

For a stochastic policy \(\mu\), \cref{lem:lfa-convex-hull} gives
\[
  \bA_\mu=
  \sum_{\pi\in\Theta}c_\pi(\mu)\bA_\pi
\]
with convex weights.  The function \(V_{\varepsilon}^\infty\) is convex because it is a pointwise limit of nondecreasing convex functions \(V_{\varepsilon}^t\).  Hence Jensen's inequality gives
\[
\begin{aligned}
V_{\varepsilon}^\infty(\bA_\mu x) \le
\sum_{\pi\in\Theta}c_\pi(\mu)V_{\varepsilon}^\infty(\bA_\pi x) \le
\beta_\varepsilon^2
\left(V_{\varepsilon}^\infty(x)-\|x\|_2^2\right).
\end{aligned}
\]
Taking square roots gives \cref{eq:lfa-p-contraction-mode}.

For arbitrary \(\theta,\bar\theta\), \cref{prop:pairwise-direct-lfa} gives
\[
  \bT_\alpha(\theta)-\bT_\alpha(\bar\theta)
  =
  \bA_{\mu_{\theta,\bar\theta}}(\theta-\bar\theta).
\]
Using \cref{eq:lfa-p-contraction-mode} gives the contraction inequality \cref{eq:lfa-T-contraction}.  Since \((\R^m,p_{\varepsilon})\) is complete, Banach's fixed-point theorem implies that \(\bT_\alpha\) has a unique fixed point \(\theta^\star\).  By the definition of \(\bT_\alpha\), this fixed point satisfies \cref{eq:pbe}.

Applying \cref{eq:lfa-strong-drift} to \(x_k=\theta_k-\theta^\star\) gives
\[
  V_{\varepsilon}^\infty(x_{k+1})
  \le
  \beta_\varepsilon^2V_{\varepsilon}^\infty(x_k).
\]
Iteration proves \cref{eq:lfa-det-V-convergence}.  The Euclidean and Q-function bounds follow from \cref{eq:lfa-Veps-equivalence} and \(\|\Phi x\|_2\le\|\Phi\|_2\|x\|_2\).
\end{proof}

\begin{figure}[t]
\centering
\includegraphics[width=0.58\textwidth]{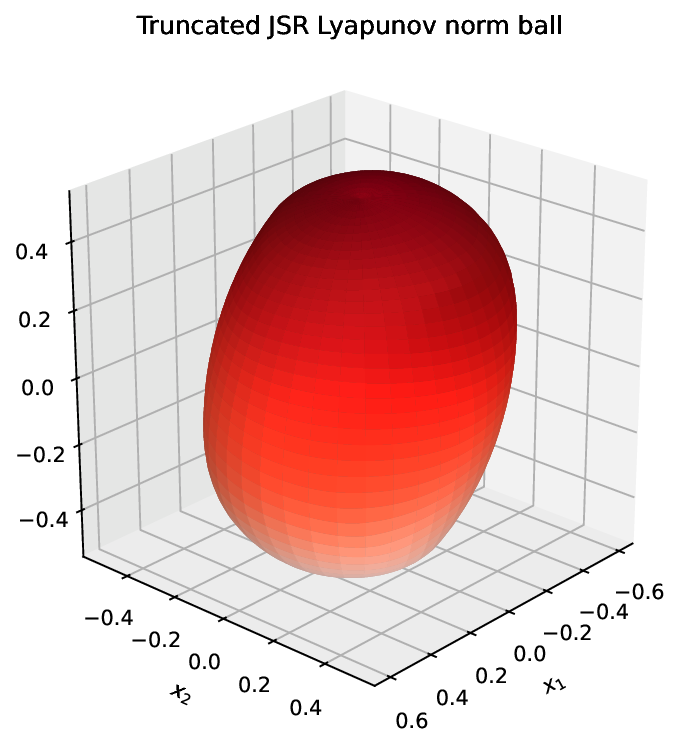}
\caption{Truncated norm ball induced by the Lyapunov construction in \cref{thm:deterministic-jsr-lfa} for a three-dimensional MDP\@. Since the feature dimension is three, the set \(\{x\in\mathbb{R}^3 : p_{\varepsilon,T}(x)\le 1\}\) can be plotted directly, without projection.}
\label{fig:thm1-norm-ball-3d}
\end{figure}

\begin{example}\label{ex:thm1-norm-ball-3d}
Consider a three-state, two-action MDP with \(\calS=\{1,2,3\}\) and \(\calA=\{1,2\}\).  Enumerate the state-action pairs in the action-block order as
\[
(1,1),(2,1),(3,1),(1,2),(2,2),(3,2)
\]
and let the transition matrix be
\[
P=
\begin{bmatrix}
0.7325 & 0.0122 & 0.2552\\
0.6359 & 0.2104 & 0.1537\\
0.5133 & 0.1950 & 0.2917\\
0.4722 & 0.0379 & 0.4899\\
0.0023 & 0.8670 & 0.1307\\
0.7437 & 0.0553 & 0.2010
\end{bmatrix},
\]
let $\gamma=0.7965$, $\alpha=0.9000$, and use the rounded sampling distribution
\[
D\approx\diag(0.1595,\,0.0198,\,0.1480,\,0.2228,\,0.2155,\,0.2343),
\]
and choose the feature matrix
\[
\Phi=
\begin{bmatrix}
-0.0957 & -0.3996 & -0.5050\\
 0.0242 &  0.1328 &  0.1858\\
 0.7378 &  0.4582 &  0.0919\\
-0.4882 &  0.5305 &  0.2531\\
-0.2158 & -0.1461 & -0.2595\\
 0.4013 &  0.5568 & -0.7554
\end{bmatrix}
\in\R^{6\times 3}.
\]
Since there are three states and two actions, there are \(2^3=8\) deterministic policies. For each deterministic policy \(\pi\in\Theta\), the corresponding direct mode is
\[
\bA_\pi
=
I-
\alpha\Phi^\top D\Phi
+
\alpha\gamma\Phi^\top DP\bPi^\pi\Phi
\in\R^{3\times 3}.
\]
For this example, the eight mode norms are approximately
\[
0.8966,\; 0.9324,\; 0.9135,\; 0.9461,\; 0.9335,\; 0.9653,\; 0.9338,\; 0.9678.
\]
Hence, for every product length \(k\), submultiplicativity gives the common induced-norm bound
\[
\begin{aligned}
\jsr(\mathcal A_\alpha)
&=
\lim_{k\to\infty}
\max_{\pi_1,\ldots,\pi_k\in\Theta}
\left\|
\bA_{\pi_k}\cdots\bA_{\pi_1}
\right\|_2^{1/k} \\
&\le
\lim_{k\to\infty}
\max_{\pi_1,\ldots,\pi_k\in\Theta}
\left(
\prod_{j=1}^{k}\|\bA_{\pi_j}\|_2
\right)^{1/k} \\
&\le
\max_{\pi\in\Theta}\|\bA_\pi\|_2
=
0.9678
<1.
\end{aligned}
\]
Therefore, \cref{thm:deterministic-jsr-lfa} applies. We choose $\beta_\varepsilon=0.975$, $T=4$, so that \(\max_{\pi\in\Theta}\|\bA_\pi\|_2<\beta_\varepsilon<1\).
For intuition, \cref{fig:thm1-norm-ball-3d} shows a three-dimensional visualization of the truncated norm ball induced by~\cref{thm:deterministic-jsr-lfa}
\[
  p_{\varepsilon,T}(x)
  :=
  \sqrt{
    \sum_{\ell=0}^{T}
    \beta_\varepsilon^{-2\ell}
    \max_{\pi_1,\ldots,\pi_\ell\in\Theta}
    \left\|
    \bA_{\pi_\ell}\cdots \bA_{\pi_1} x
    \right\|_2^2
  },
\]
which approximates the norm \(p_\varepsilon(x)=\sqrt{V_\varepsilon^\infty(x)}\) from~\cref{thm:deterministic-jsr-lfa}.
The resulting surface in~\cref{fig:thm1-norm-ball-3d} is visibly non-Euclidean: the maximizing switched products produce a slightly bumpy but smooth curved norm ball rather than a sphere.
\end{example}

The previous theorem gives convergence of deterministic linear Q-learning from \(\jsr(\mathcal A_\alpha)<1\). The same certificate also yields the following arbitrary-switching inclusion result.

\begin{theorem}\label{thm:direct-inclusion-lfa}
Consider the direct switched linear inclusion
\begin{equation}\label{eq:lfa-inclusion}
  x_{k+1}=\bA_kx_k,
  \qquad
  \bA_k\in\co(\mathcal A_\alpha),
  \qquad k\in\{0,1,\ldots\}.
\end{equation}
If
\[
  \jsr(\mathcal A_\alpha)<1,
\]
then \cref{eq:lfa-inclusion} is uniformly exponentially stable under arbitrary switching.  More specifically, for every \(\varepsilon>0\) with \(\jsr(\mathcal A_\alpha)+\varepsilon<1\), the norm \(p_{\varepsilon}\) defined by \cref{eq:lfa-Veps-infty} satisfies
\[
  p_{\varepsilon}(Ax)
  \le
  (\jsr(\mathcal A_\alpha)+\varepsilon)p_{\varepsilon}(x),
  \qquad
  \forall \bA\in\co(\mathcal A_\alpha),
  \quad
  \forall x\in\R^m.
\]
Consequently, there exist \(C\ge1\) and \(\eta\in(0,1)\) such that
\[
  \|\bA_{k-1}\cdots \bA_0x\|_2
  \le
  C\eta^k\|x\|_2
\]
for every \(k\ge0\), every \(x\in\R^m\), and every switching sequence \(\bA_j\in\co(\mathcal A_\alpha)\).
If instead \(\jsr(\mathcal A_\alpha)>1\), then the arbitrary-switching inclusion admits a periodically switched trajectory whose state does not converge to zero.
\end{theorem}

\begin{proof}
Because of \cref{lem:lfa-convex-hull}, the JSR of \(\co(\mathcal A_\alpha)\) equals \(\jsr(\mathcal A_\alpha)\).

Assume first that \(\jsr(\mathcal A_\alpha)<1\).  Fix any \(\varepsilon>0\) with \(\beta_\varepsilon:=\jsr(\mathcal A_\alpha)+\varepsilon<1\).  The Lyapunov construction in \cref{thm:deterministic-jsr-lfa} gives
\[
  V_\varepsilon^\infty(\bA_\pi x)
  \le
  \beta_\varepsilon^2
  \left(V_\varepsilon^\infty(x)-\|x\|_2^2\right)
\]
for every deterministic policy \(\pi\in\Theta\).  Now take an arbitrary \(\bA\in\co(\mathcal A_\alpha)\) and write
\[
  \bA=\sum_{\pi\in\Theta}c_\pi\bA_\pi,
  \qquad
  c_\pi\ge0,
  \qquad
  \sum_{\pi\in\Theta}c_\pi=1.
\]
Since \(V_\varepsilon^\infty\) is convex, Jensen's inequality gives
\[
\begin{aligned}
V_\varepsilon^\infty(\bA x)
&=
V_\varepsilon^\infty\left(\sum_{\pi\in\Theta}c_\pi\bA_\pi x\right)\\
&\le
\sum_{\pi\in\Theta}c_\pi V_\varepsilon^\infty(\bA_\pi x)\\
&\le
\beta_\varepsilon^2
\left(V_\varepsilon^\infty(x)-\|x\|_2^2\right)
\le
\beta_\varepsilon^2V_\varepsilon^\infty(x).
\end{aligned}
\]
Taking square roots yields
\[
  p_{\varepsilon}(\bA x)
  \le
  \beta_\varepsilon p_{\varepsilon}(x)
  =
  (\jsr(\mathcal A_\alpha)+\varepsilon)p_{\varepsilon}(x)
\]
for every \(\bA\in\co(\mathcal A_\alpha)\).  Iterating this inequality gives exponential decay in the \(p_\varepsilon\)-norm.  Since all norms on \(\R^m\) are equivalent, the same decay holds in Euclidean norm with adjusted constants, proving uniform exponential stability of \cref{eq:lfa-inclusion}.

Assume next that \(\jsr(\mathcal A_\alpha)>1\).  By the Berger--Wang formula for finite bounded matrix families~\cite{jungers2009joint}, there exist matrices \(\bA_1,\ldots,\bA_\ell\in\co(\mathcal A_\alpha)\) such that the product
\[
\mathbf{P}_\sigma:=\bA_\ell\cdots \bA_1
\]
satisfies \(\rho(\mathbf{P}_\sigma)>1\).  Choose an initial vector \(x_0\) in an invariant real subspace associated with an eigenvalue of modulus larger than one.  Repeating the word \(\sigma=(\bA_1,\ldots,\bA_\ell)\) periodically gives
\[
  x_{n\ell}=\mathbf{P}_\sigma^n x_0,
\]
which does not converge to zero.  Hence the arbitrary-switching inclusion admits a periodically switched nonconvergent trajectory.
\end{proof}

The arbitrary-switching inclusion in \cref{eq:lfa-inclusion} is a conservative worst-case model: it allows switching sequences that may never be generated by the Bellman maximum along an actual trajectory. The actual deterministic nonlinear recursion \(\theta_{k+1}=\bT_\alpha(\theta_k)\) realizes only those stochastic policies generated by Bellman maximization along a trajectory. Therefore \(\jsr(\mathcal A_\alpha)<1\) is used here as a sufficient condition. A value \(\jsr(\mathcal A_\alpha)>1\) shows that the larger arbitrary-switching inclusion can be unstable, but it does not by itself show nonconvergence of the Bellman-generated nonlinear recursion.

The preceding discussion shows that the arbitrary-switching inclusion and the actual nonlinear recursion should not be identified. The following example makes this distinction explicit: the direct JSR is larger than one, so the arbitrary-switching inclusion can be unstable, but the Bellman-generated nonlinear recursion is globally convergent.
\begin{example}\label{ex:jsr-greater-than-one-convergent}
Consider a one-state MDP with \(\calS=\{1\}\), two actions \(\calA=\{1,2\}\), zero reward, and deterministic self-transition, $P(1\mid 1,a)=1$, $R(1,a)=0$. Let $\gamma=0.9$, $\alpha=0.9$, $d(1,1)=0.9$, $d(1,2)=0.1$, and use the one-dimensional feature representation
\[
\phi(1,1)=1,
\qquad
\phi(1,2)=-2,
\qquad
\Phi=\begin{bmatrix}1\\ -2\end{bmatrix}.
\]
Then
\[
\Phi^\top D\Phi
=0.9\cdot1^2+0.1\cdot(-2)^2
=1.3,
\qquad
\Phi^\top DP
=0.9\cdot1+0.1\cdot(-2)
=0.7.
\]
There are two deterministic policies, selecting action 1 or action 2.  The corresponding scalar direct modes are
\[
\bA_1
=1-0.9\cdot1.3+0.9\cdot0.9\cdot0.7\cdot1
=0.397,
\qquad
\bA_2
=1-0.9\cdot1.3+0.9\cdot0.9\cdot0.7\cdot(-2)
=-1.304.
\]
Hence
\[
\jsr(\mathcal A_\alpha)
=
\rho(\{0.397,-1.304\})
=1.304>1.
\]
The arbitrary-switching inclusion can therefore generate the divergent trajectory \(x_k=(-1.304)^kx_0\) by repeatedly selecting the second mode.

The actual deterministic nonlinear recursion behaves differently.  Since rewards are zero,
\[
V_\theta
=\max\{\theta,-2\theta\},
\]
and the mean update is
\[
\bT_\alpha(\theta)
=\theta+\alpha\Phi^\top D\left(\gamma PV_\theta-\Phi\theta\right)
=-0.17\theta+0.567\max\{\theta,-2\theta\}.
\]
Thus
\[
\bT_\alpha(\theta)
=
\begin{cases}
0.397\theta, & \theta\ge0,\\
-1.304\theta, & \theta<0.
\end{cases}
\]
If \(\theta_0\ge0\), then \(\theta_k=0.397^k\theta_0\to0\).  If \(\theta_0<0\), then \(\theta_1=-1.304\theta_0>0\), and from that point onward
\[
\theta_k=0.397^{k-1}\theta_1\to0.
\]
Therefore the actual deterministic nonlinear recursion converges globally to the unique fixed point \(\theta^\star=0\), even though \(\jsr(\mathcal A_\alpha)>1\).

The reason is that the unstable second mode is not repeatedly admissible under Bellman maximization. It is selected only when \(\theta<0\), and one application sends the iterate to the region \(\theta>0\), where the stable first mode is selected. Therefore, the product \(\bA_2\bA_2\bA_2\cdots\), although allowed by the arbitrary-switching inclusion and counted by the JSR, is not realized by the actual Bellman-generated dynamics.
\end{example}

\section{Finite-Time Error Analysis of Linear Q-Learning}\label{sec:iid-lfa}
This section analyzes the sampled linear Q-learning recursion under an independent and identically distributed (i.i.d.) observation model. Unlike the deterministic analysis, the stochastic recursion contains persistent martingale-difference noise. For clarity, we impose an i.i.d.\ observation model; the same switching decomposition can be extended to Markovian observation models using the techniques of~\cite{lee2026lyapunovcertified}.

\subsection{Analysis with a Boundedness Assumption}
At each time \(k\), we sample \((s_k,a_k)\) independently according to \(d\), then sample \(s'_k\sim P(\cdot\mid s_k,a_k)\), and set \(r_{k+1}:=r(s_k,a_k,s'_k)\). Define the filtration explicitly by
\[
\calF_0:=\sigma(\theta_0),
\qquad
\calF_k:=\sigma\left(\theta_0,\{(s_t,a_t,s'_t,r_{t+1}):0\le t\le k-1\}\right),
\quad k\ge1 .
\]
Under this sampling model, we consider the following linear Q-learning update.

\begin{equation}\label{eq:iid-lfa-update}
\theta_{k+1}
=
\theta_k+
\alpha\phi(s_k,a_k)
\left(
 r_{k+1}
 +
 \gamma\max_{u\in\calA}\phi(s'_k,u)^\top\theta_k
 -
 \phi(s_k,a_k)^\top\theta_k
\right),
\qquad k\in\{0,1,\ldots\}.
\end{equation}

We first assume that the stochastic iterates remain in a fixed \(\ell_\infty\)-ball. This assumption isolates the main switching-system argument by keeping the martingale-difference noise uniformly controlled. A projection-free small-step-size analysis later in this section removes this assumption.

\begin{assumption}\label{ass:iid-boundedness}
For some fixed radius \(B_\theta>0\), the iterates of \cref{eq:iid-lfa-update} satisfy
\begin{equation}\label{eq:iid-boundedness-assumption}
  \|\theta_k\|_\infty\le B_\theta,
  \qquad \forall k\ge0,
  \quad\text{a.s.}
\end{equation}
\end{assumption}
To proceed, let us define the sample update vector
\begin{equation*}
\widehat g_k(\theta)
:=
\phi(s_k,a_k)
\left(
 r_{k+1}
 +
 \gamma\max_{u\in\calA}\phi(s'_k,u)^\top\theta
 -
 \phi(s_k,a_k)^\top\theta
\right).
\end{equation*}
Then
\begin{equation}\label{eq:mean-g-iid}
  \E[\widehat g_k(\theta_k)\mid\calF_k]
  =
  g(\theta_k),
\end{equation}
where \(g\) is defined in \cref{eq:g-theta-def}. Moreover, let us define the martingale-difference noise
\begin{equation}\label{eq:lfa-wk-def}
  w_k
  :=
  \widehat g_k(\theta_k)-g(\theta_k).
\end{equation}
Then \(\E[w_k\mid\calF_k]=0\), and \cref{eq:iid-lfa-update} is
\begin{equation}\label{eq:iid-lfa-vector}
  \theta_{k+1}
  =
  \bT_\alpha(\theta_k)+\alpha w_k,
  \qquad k\in\{0,1,\ldots\}.
\end{equation}

The exact switching linear system representation applies to this recursion.
\begin{proposition}\label{prop:iid-direct-lfa}
Suppose that \(\jsr(\mathcal A_\alpha)<1\), and let \(\theta^\star\) be the unique projected Bellman fixed point from \cref{thm:deterministic-jsr-lfa}. Define
\[
  x_k:=\theta_k-\theta^\star .
\]
Then, for each \(k\in\{0,1,\ldots\}\), there exists an \(\calF_k\)-measurable stochastic policy \(\mu_k\) such that
\begin{equation}\label{eq:iid-direct-error-lfa}
  x_{k+1}
  =
  \bA_{\mu_k} x_k+
  \alpha w_k,
  \qquad k\in\{0,1,\ldots\},
\end{equation}
where \(\E[w_k\mid\calF_k]=0\).
\end{proposition}

\begin{proof}
From \cref{eq:iid-lfa-vector},
\[
  \theta_{k+1}=\bT_\alpha(\theta_k)+\alpha w_k,
  \qquad k\in\{0,1,\ldots\}.
\]
Subtract \(\theta^\star=\bT_\alpha(\theta^\star)\) from both sides:
\[
  x_{k+1}
  =
  \bT_\alpha(\theta_k)-\bT_\alpha(\theta^\star)+\alpha w_k.
\]
By \cref{prop:pairwise-direct-lfa}, there exists an \(\calF_k\)-measurable stochastic policy \(\mu_k\) such that
\[
  \bT_\alpha(\theta_k)-\bT_\alpha(\theta^\star)=\bA_{\mu_k}x_k.
\]
This proves the direct identity.  The identity \(\E[w_k\mid\calF_k]=0\) follows from \cref{eq:mean-g-iid,eq:lfa-wk-def}.
\end{proof}

\subsection{Finite-time error bound}
We now turn to the finite-time error analysis of linear Q-learning.
Fix \(\varepsilon>0\) such that \(\beta_\varepsilon:=\jsr(\mathcal A_\alpha)+\varepsilon<1\), and define the product-growth constant
\begin{equation}\label{eq:iid-Kbeta-def}
K_{\beta_\varepsilon}
:=
\sup_{\ell\ge0}
\beta_\varepsilon^{-\ell}
\sup_{\mathbf{M}_0,\ldots,\mathbf{M}_{\ell-1}\in\co(\mathcal A_\alpha)}
\left\|
\mathbf{M}_{\ell-1}\cdots \mathbf{M}_0
\right\|_2 .
\end{equation}
We start with three elementary estimates used in the finite-time bound.
\begin{lemma}\label{lem:iid-product-growth-finite}
Suppose that \(\jsr(\mathcal A_\alpha)<1\).  Fix \(\varepsilon>0\) such that \(\beta_\varepsilon:=\jsr(\mathcal A_\alpha)+\varepsilon<1\), and let \(K_{\beta_\varepsilon}\) be defined by \cref{eq:iid-Kbeta-def}.  Then \(K_{\beta_\varepsilon}<\infty\).
\end{lemma}

\begin{proof}
By \cref{lem:lfa-convex-hull}, \(\rho(\co(\mathcal A_\alpha))=\jsr(\mathcal A_\alpha)\). Since \(\beta_\varepsilon>\rho(\co(\mathcal A_\alpha))\), the definition of the JSR implies that there exist constants \(C<\infty\) and \(\eta\in(\rho(\co(\mathcal A_\alpha)),\beta_\varepsilon)\) such that every length-\(\ell\) product generated by \(\co(\mathcal A_\alpha)\) has Euclidean norm at most \(C\eta^\ell\).  Therefore
\[
\beta_\varepsilon^{-\ell}
\sup_{\mathbf{M}_0,\ldots,\mathbf{M}_{\ell-1}\in\co(\mathcal A_\alpha)}
\left\|
\mathbf{M}_{\ell-1}\cdots \mathbf{M}_0
\right\|_2
\le
C\left(\frac{\eta}{\beta_\varepsilon}\right)^\ell
\le C,
\]
for all \(\ell\ge0\).  Taking the supremum over \(\ell\) proves \(K_{\beta_\varepsilon}<\infty\).
\end{proof}

To proceed further, let us define
\begin{equation}\label{eq:iid-L-def}
L
:=
\sup_{\mu,\nu}
\left\|
\Phi^\top D P(\bPi^\mu-\bPi^\nu)\Phi
\right\|_2,
\end{equation}
where the supremum is over all stochastic policies \(\mu,\nu:\calS\to\Delta_{|\calA|}\).

\begin{lemma}\label{lem:iid-mode-difference-bound}
For stochastic policies \(\mu\) and \(\nu\),
\begin{equation}\label{eq:iid-mode-difference-bound}
  \|\bA_\mu-\bA_\nu\|_2
  \le
  \alpha\gamma L .
\end{equation}
\end{lemma}

\begin{proof}
The definition of \(\bA_\mu\) gives
\[
  \bA_\mu-\bA_\nu
  =
  \alpha\gamma\Phi^\top DP(\bPi^\mu-\bPi^\nu)\Phi .
\]
Taking Euclidean operator norms and using the definition of \(L\) in \cref{eq:iid-L-def} gives \cref{eq:iid-mode-difference-bound}.
\end{proof}

\begin{lemma}\label{lem:iid-projected-noise-bound}
Assume \cref{ass:iid-boundedness}.  Then, for \(w_k\) defined in \cref{eq:lfa-wk-def},
\[
  \E[w_k\mid\calF_k]=0,
  \qquad
  \E[\|w_k\|_2^2\mid\calF_k]
  \le
  4\phi_{\max}^2\left(R_{\max}+(1+\gamma)\max_{(s,a)\in\calS\times\calA}\|\phi(s,a)\|_1\,B_\theta\right)^2,
  \qquad \forall k\ge0.
\]
\end{lemma}

\begin{proof}
The assumed bound \cref{eq:iid-boundedness-assumption} implies
\[
  \|\Phi\theta_k\|_\infty
  \le
  \max_{(s,a)\in\calS\times\calA}\|\phi(s,a)\|_1\,B_\theta.
\]
Therefore, for every sample,
\[
\begin{aligned}
\|\widehat g_k(\theta_k)\|_2
&\le
\|\phi(s_k,a_k)\|_2
\left(
R_{\max}
+
\gamma\|\Phi\theta_k\|_\infty
+
\|\Phi\theta_k\|_\infty
\right)\\
&\le
\phi_{\max}
\left(
R_{\max}+(1+\gamma)
\max_{\substack{(s,a)\in\\\calS\times\calA}}
\|\phi(s,a)\|_1\,B_\theta
\right).
\end{aligned}
\]
Since \(g(\theta_k)=\E[\widehat g_k(\theta_k)\mid\calF_k]\), Jensen's inequality gives the same bound for \(\|g(\theta_k)\|_2\).  Hence
\[
\begin{aligned}
  \|w_k\|_2
  &\le
  \|\widehat g_k(\theta_k)\|_2+
  \|g(\theta_k)\|_2\\
  &\le
  2\phi_{\max}
  \left(
  R_{\max}+(1+\gamma)
  \max_{\substack{(s,a)\in\\\calS\times\calA}}
  \|\phi(s,a)\|_1\,B_\theta
  \right).
\end{aligned}
\]
Squaring and taking conditional expectation gives
\[
\begin{aligned}
  \E[\|w_k\|_2^2\mid\calF_k]
  \le
  4\phi_{\max}^2
  \left(
  R_{\max}+(1+\gamma)
  \max_{\substack{(s,a)\in\\\calS\times\calA}}
  \|\phi(s,a)\|_1\,B_\theta
  \right)^2.
\end{aligned}
\]
The martingale-difference property follows from \cref{eq:mean-g-iid,eq:lfa-wk-def}.
\end{proof}

The next theorem combines the preceding estimates to obtain a finite-time error bound under~\cref{ass:iid-boundedness}.
\begin{theorem}\label{thm:iid-jsr-lfa}
Suppose that \(\jsr(\mathcal A_\alpha)<1\).  Fix \(\varepsilon>0\) such that \(\beta_\varepsilon:=\jsr(\mathcal A_\alpha)+\varepsilon<1\), and let \(K_{\beta_\varepsilon}\) be defined in \cref{eq:iid-Kbeta-def}.  Let \(\theta^\star\) be the unique projected Bellman fixed point from \cref{thm:deterministic-jsr-lfa}.  Assume that \cref{ass:iid-boundedness} holds for some \(B_\theta>0\).  Fix an arbitrary stochastic policy \(\bar\mu\).  Then the linear Q-learning recursion \cref{eq:iid-lfa-update} satisfies, for all \(k\ge0\),
\begin{equation}\label{eq:iid-theta-bound}
\begin{aligned}
\E[\|\theta_k-\theta^\star\|_2]
\le\;&
K_{\beta_\varepsilon}\beta_\varepsilon^k\|\theta_0-\theta^\star\|_2 \\
&+
\alpha\gamma K_{\beta_\varepsilon}^2 k\beta_\varepsilon^{k-1}
\|\theta_0-\theta^\star\|_2
L \\
&+
\left(
1+
\frac{\alpha\gamma K_{\beta_\varepsilon}}{1-\beta_\varepsilon}
L
\right) \\
&\qquad\times
\frac{2\alpha K_{\beta_\varepsilon}\phi_{\max}}{\sqrt{1-\beta_\varepsilon^2}}
\left(
R_{\max}+(1+\gamma)
\max_{\substack{(s,a)\in\\\calS\times\calA}}
\|\phi(s,a)\|_1\,B_\theta
\right),
\end{aligned}
\end{equation}
with the convention that \(k\beta_\varepsilon^{k-1}=0\) when \(k=0\).  Consequently,

\begin{equation}\label{eq:iid-limsup-bound}
\begin{aligned}
\limsup_{k\to\infty}
\E[\|\theta_k-\theta^\star\|_2]
\le\;&
\left(
1+
\frac{\alpha\gamma K_{\beta_\varepsilon}}{1-\beta_\varepsilon}
L
\right) \\
&\times
\frac{2\alpha K_{\beta_\varepsilon}\phi_{\max}}{\sqrt{1-\beta_\varepsilon^2}}
\left(
R_{\max}+(1+\gamma)
\max_{\substack{(s,a)\in\\\calS\times\calA}}
\|\phi(s,a)\|_1\,B_\theta
\right).
\end{aligned}
\end{equation}
\end{theorem}

\begin{proof}
Set \(x_k:=\theta_k-\theta^\star\).  By \cref{prop:iid-direct-lfa}, the error process satisfies the exact switched recursion
\[
  x_{k+1}=\bA_{\mu_k}x_k+\alpha w_k,
  \qquad
  \E[w_k\mid\calF_k]=0,
\]
for an \(\calF_k\)-measurable stochastic policy \(\mu_k\).  By \cref{lem:lfa-convex-hull}, each stochastic-policy mode lies in \(\co(\mathcal A_\alpha)\). Hence the definition of \(K_{\beta_\varepsilon}\) implies the pathwise product bound
\[
  \|\bA_{\nu_{\ell-1}}\cdots \bA_{\nu_0}\|_2
  \le
  K_{\beta_\varepsilon}\beta_\varepsilon^\ell,
  \qquad
  \ell\ge0,
\]
for every finite sequence of stochastic policies \(\nu_0,\ldots,\nu_{\ell-1}\), with the empty product interpreted as the identity.

We next derive the fixed-policy reference-filter estimates used below.  Fix an arbitrary stochastic policy \(\bar\mu\).  Define the fixed-policy reference filter
\begin{equation}\label{eq:iid-reference-filter-app}
  y_{k+1}
  =
  \bA_{\bar\mu} y_k+
  \alpha w_k,
  \qquad
  y_0=x_0.
\end{equation}
Split
\[
  y_k=\bar y_k+\eta_k,
\]
where
\[
  \bar y_{k+1}=\bA_{\bar\mu}\bar y_k,
  \quad \bar y_0=x_0,
  \qquad
  \eta_{k+1}=\bA_{\bar\mu}\eta_k+
  \alpha w_k,
  \quad \eta_0=0.
\]
Since \(\bA_{\bar\mu}\in\co(\mathcal A_\alpha)\), the preceding product bound gives
\begin{equation}\label{eq:iid-reference-deterministic-bound}
  \|\bar y_k\|_2
  \le
  K_{\beta_\varepsilon}\beta_\varepsilon^k\|x_0\|_2 .
\end{equation}
For the noise part,
\[
  \eta_k
  =
  \alpha\sum_{t=0}^{k-1}\bA_{\bar\mu}^{k-1-t}w_t .
\]
Expanding the squared norm gives
\[
\begin{aligned}
\E[\|\eta_k\|_2^2]
&=
\alpha^2
\sum_{s=0}^{k-1}\sum_{t=0}^{k-1}
\E\!\left[
 w_s^\top
 (\bA_{\bar\mu}^{k-1-s})^\top
 \bA_{\bar\mu}^{k-1-t}w_t
\right].
\end{aligned}
\]
If \(s<t\), then
\(w_s^\top(\bA_{\bar\mu}^{k-1-s})^\top\bA_{\bar\mu}^{k-1-t}\) is \(\calF_t\)-measurable. Therefore
\[
\begin{aligned}
&\E\!\left[
 w_s^\top
 (\bA_{\bar\mu}^{k-1-s})^\top
 \bA_{\bar\mu}^{k-1-t}w_t
\right] \\
&\qquad =
\E\!\left[
 w_s^\top
 (\bA_{\bar\mu}^{k-1-s})^\top
 \bA_{\bar\mu}^{k-1-t}
 \E[w_t\mid\calF_t]
\right]
=0 .
\end{aligned}
\]
The case \(t<s\) is identical after conditioning on \(\calF_s\). Thus only the diagonal terms remain, and \cref{lem:iid-projected-noise-bound} gives
\[
\begin{aligned}
\E[\|\eta_k\|_2^2]
&=
\alpha^2
\sum_{t=0}^{k-1}
\E\!\left[
\|\bA_{\bar\mu}^{k-1-t}w_t\|_2^2
\right] \\
&\le
\alpha^2
\sum_{t=0}^{k-1}
\|\bA_{\bar\mu}^{k-1-t}\|_2^2
\E[\|w_t\|_2^2]\\
&\le
4\alpha^2K_{\beta_\varepsilon}^2\phi_{\max}^2
\left(
R_{\max}+(1+\gamma)
\max_{\substack{(s,a)\in\\\calS\times\calA}}
\|\phi(s,a)\|_1\,B_\theta
\right)^2
\sum_{j=0}^{k-1}\beta_\varepsilon^{2j}\\
&\le
\frac{4\alpha^2K_{\beta_\varepsilon}^2\phi_{\max}^2}{1-\beta_\varepsilon^2}
\left(
R_{\max}+(1+\gamma)
\max_{\substack{(s,a)\in\\\calS\times\calA}}
\|\phi(s,a)\|_1\,B_\theta
\right)^2.
\end{aligned}
\]
Here the third line also uses the tower property to pass from the conditional second-moment bound in \cref{lem:iid-projected-noise-bound} to an unconditional bound.  By Jensen's inequality,
\begin{equation}\label{eq:iid-reference-noise-bound}
\begin{aligned}
  \E[\|\eta_k\|_2]
  \le
  \frac{2\alpha K_{\beta_\varepsilon}\phi_{\max}}{\sqrt{1-\beta_\varepsilon^2}}
  \left(
  R_{\max}+(1+\gamma)
  \max_{\substack{(s,a)\in\\\calS\times\calA}}
  \|\phi(s,a)\|_1\,B_\theta
  \right),
  \qquad \forall k\ge0.
\end{aligned}
\end{equation}

Define the residual
\[
  r_k:=x_k-y_k .
\]
Subtracting \cref{eq:iid-reference-filter-app} from \cref{eq:iid-direct-error-lfa} gives the noise-free residual recursion
\begin{equation}\label{eq:iid-residual-recursion-app}
  r_{k+1}
  =
  \bA_{\mu_k}r_k+
  (\bA_{\mu_k}-\bA_{\bar\mu})y_k,
  \qquad r_0=0.
\end{equation}
The stochastic term \(\alpha w_k\) cancels pathwise.  Now split \(y_k=\bar y_k+\eta_k\) and write
\[
  r_k=u_k+v_k,
\]
where
\[
  u_{k+1}=\bA_{\mu_k}u_k+(\bA_{\mu_k}-\bA_{\bar\mu})\bar y_k,
  \qquad u_0=0,
\]
and
\[
  v_{k+1}=\bA_{\mu_k}v_k+(\bA_{\mu_k}-\bA_{\bar\mu})\eta_k,
  \qquad v_0=0.
\]
For \(k=0\), both residual sums below are empty.  For \(k\ge1\), iterating the recursion for \(u_k\) gives
\[
  u_k
  =
  \sum_{i=0}^{k-1}
  \bA_{\mu_{k-1}}\cdots \bA_{\mu_{i+1}}
  (\bA_{\mu_i}-\bA_{\bar\mu})\bar y_i,
\]
where the product is the identity when \(i=k-1\).  Using the product bound, \cref{eq:iid-mode-difference-bound}, and \cref{eq:iid-reference-deterministic-bound},
\[
\begin{aligned}
\|u_k\|_2
&\le
\sum_{i=0}^{k-1}
\left\|\bA_{\mu_{k-1}}\cdots \bA_{\mu_{i+1}}\right\|_2
\left\|\bA_{\mu_i}-\bA_{\bar\mu}\right\|_2
\|\bar y_i\|_2 \\
&\le
\sum_{i=0}^{k-1}
K_{\beta_\varepsilon}\beta_\varepsilon^{k-1-i}
\alpha \gamma L
K_{\beta_\varepsilon}\beta_\varepsilon^i\|x_0\|_2\\
&=
\alpha \gamma L K_{\beta_\varepsilon}^2
k\beta_\varepsilon^{k-1}\|x_0\|_2 .
\end{aligned}
\]
The same iteration for \(v_k\) gives
\[
  v_k
  =
  \sum_{i=0}^{k-1}
  \bA_{\mu_{k-1}}\cdots \bA_{\mu_{i+1}}
  (\bA_{\mu_i}-\bA_{\bar\mu})\eta_i .
\]
Therefore, by the product bound, \cref{eq:iid-mode-difference-bound}, and \cref{eq:iid-reference-noise-bound},
\[
\begin{aligned}
\E[\|v_k\|_2]
&\le
\sum_{i=0}^{k-1}
K_{\beta_\varepsilon}\beta_\varepsilon^{k-1-i}
\alpha\gamma
L
\E[\|\eta_i\|_2]\\
&\le
\alpha\gamma K_{\beta_\varepsilon}
L
\frac{2\alpha K_{\beta_\varepsilon}\phi_{\max}}{\sqrt{1-\beta_\varepsilon^2}}\\
&\qquad\times
\left(
R_{\max}+(1+\gamma)
\max_{\substack{(s,a)\in\\\calS\times\calA}}
\|\phi(s,a)\|_1\,B_\theta
\right)
\sum_{j=0}^{k-1}\beta_\varepsilon^j\\
&\le
\frac{\alpha\gamma K_{\beta_\varepsilon}}{1-\beta_\varepsilon}
L
\frac{2\alpha K_{\beta_\varepsilon}\phi_{\max}}{\sqrt{1-\beta_\varepsilon^2}}\\
&\qquad\times
\left(
R_{\max}+(1+\gamma)
\max_{\substack{(s,a)\in\\\calS\times\calA}}
\|\phi(s,a)\|_1\,B_\theta
\right).
\end{aligned}
\]
Finally,
\[
  x_k=y_k+r_k=\bar y_k+\eta_k+u_k+v_k.
\]
Taking norms and expectations, and then using \cref{eq:iid-reference-deterministic-bound,eq:iid-reference-noise-bound} and the two residual bounds, gives
\[
\begin{aligned}
\E[\|x_k\|_2]
\le\;&
K_{\beta_\varepsilon}\beta_\varepsilon^k\|x_0\|_2 \\
&+
\frac{2\alpha K_{\beta_\varepsilon}\phi_{\max}}{\sqrt{1-\beta_\varepsilon^2}}
\left(
R_{\max}+(1+\gamma)
\max_{\substack{(s,a)\in\\\calS\times\calA}}
\|\phi(s,a)\|_1\,B_\theta
\right)\\
&+
\alpha\gamma K_{\beta_\varepsilon}^2 k\beta_\varepsilon^{k-1}\|x_0\|_2
L\\
&+
\frac{\alpha\gamma K_{\beta_\varepsilon}}{1-\beta_\varepsilon}
L
\frac{2\alpha K_{\beta_\varepsilon}\phi_{\max}}{\sqrt{1-\beta_\varepsilon^2}}\\
&\qquad\times
\left(
R_{\max}+(1+\gamma)
\max_{\substack{(s,a)\in\\\calS\times\calA}}
\|\phi(s,a)\|_1\,B_\theta
\right).
\end{aligned}
\]
Since \(x_k=\theta_k-\theta^\star\) and \(x_0=\theta_0-\theta^\star\), this is exactly \cref{eq:iid-theta-bound}.  The \(\limsup\) bound in \cref{eq:iid-limsup-bound} follows because \(\beta_\varepsilon^k\to0\) and \(k\beta_\varepsilon^{k-1}\to0\) as \(k\to\infty\).
\end{proof}

The bound in~\cref{eq:iid-theta-bound} contains a stochastic term that does not decay with \(k\), namely the right-hand side of~\cref{eq:iid-limsup-bound}. A key question is whether this level can be made small by choosing a smaller step-size \(\alpha\). This is not immediate from~\cref{eq:iid-limsup-bound}, because the quantities \(K_{\beta_\varepsilon}\) and \(1-\beta_\varepsilon\) may themselves depend on \(\alpha\). The next subsection resolves this issue under a fixed-data condition.

\subsection{Analysis of Noise Floor}\label{subsec:fixed-data-alpha-floor}

As just discussed, we now study whether the stochastic term in~\cref{eq:iid-limsup-bound} is controlled by the step-size. We first fix the problem data independently of \(\alpha\).
\begin{assumption}\label{ass:fixed-data-alpha}
The MDP transition kernel, rewards, discount factor, feature matrix, and sampling distribution are fixed independently of \(\alpha\).  Equivalently, \(P\), \(R\), \(\gamma\), \(\Phi\), and \(D\) do not depend on \(\alpha\).
\end{assumption}

Under~\cref{ass:fixed-data-alpha}, the dependence of each deterministic mode on the step-size is affine. Specifically, each deterministic direct mode can be written as
\[
\bA_\pi(\alpha)=I+\alpha H_\pi,
\qquad
H_\pi:=-\Phi^\top D\Phi+\gamma\Phi^\top D P\bPi^\pi\Phi,
\]
where \(H_\pi\) is independent of \(\alpha\).  Hence
\[
  \mathcal A_\alpha=\{I+\alpha H_\pi:\pi\in\Theta\}.
\]
The next lemma shows that, for such a switching family which is affine in $\alpha$, stability at one step-size implies a uniform linear contraction gap for every smaller step-size.

\begin{lemma}\label{lem:fixed-data-interpolation}
Suppose that \cref{ass:fixed-data-alpha} holds and that there exists \(\bar\alpha>0\) such that
\[
  \jsr(\mathcal A_{\bar\alpha})<1.
\]
Then there exist a norm \(p\) on \(\R^m\) and a constant \(c>0\), both independent of \(\alpha\), such that for every \(0<\alpha\le\bar\alpha\),
\[
  p(\mathbf M x)
  \le
  (1-c\alpha)p(x),
  \qquad
  \forall x\in\R^m,
  \quad
  \forall \mathbf M\in\co(\mathcal A_\alpha).
\]
\end{lemma}

\begin{proof}
Choose \(\bar\beta\) such that
\[
  \jsr(\mathcal A_{\bar\alpha})<\bar\beta<1.
\]
By the product norm construction in \cref{lem:common-lyapunov-construction}, applied to the finite family \(\mathcal A_{\bar\alpha}\), there is a norm \(p\) such that
\[
  p(\bA_\pi(\bar\alpha)x)
  \le
  \bar\beta p(x),
  \qquad
  \forall x\in\R^m,
  \quad
  \forall \pi\in\Theta .
\]
By convexity of the norm, the same bound holds for every \(\mathbf B\in\co(\mathcal A_{\bar\alpha})\):
\[
  p(\mathbf Bx)\le \bar\beta p(x),
  \qquad
  \forall x\in\R^m .
\]
Fix \(0<\alpha\le\bar\alpha\), and set
\[
  t:=\frac{\alpha}{\bar\alpha}\in(0,1].
\]
For each policy \(\pi\), the fixed-data form gives
\[
\begin{aligned}
  \bA_\pi(\alpha)
  &=I+\alpha H_\pi \\
  &=I+t\bar\alpha H_\pi \\
  &=(1-t)I+t\bA_\pi(\bar\alpha).
\end{aligned}
\]
Therefore, every \(\mathbf M\in\co(\mathcal A_\alpha)\) can be written as
\[
  \mathbf M=(1-t)I+t\mathbf B
\]
for some \(\mathbf B\in\co(\mathcal A_{\bar\alpha})\).  Hence
\[
\begin{aligned}
  p(\mathbf Mx)
  &\le
  (1-t)p(x)+t p(\mathbf Bx) \\
  &\le
  \{(1-t)+t\bar\beta\}p(x) \\
  &=
  \left(1-\frac{1-\bar\beta}{\bar\alpha}\alpha\right)p(x).
\end{aligned}
\]
Define
\[
  c:=\frac{1-\bar\beta}{\bar\alpha}>0.
\]
This proves the claimed one-step bound.
\end{proof}
The interpolation lemma yields a uniform contraction in the norm \(p\). The next lemma converts this norm contraction into a Euclidean product bound with constants independent of the step-size.
\begin{lemma}\label{lem:fixed-data-uniform-product-bound}
Suppose that \cref{ass:fixed-data-alpha} holds and that there exists \(\bar\alpha>0\) such that
\[
  \jsr(\mathcal A_{\bar\alpha})<1.
\]
Then there exist constants \(c>0\) and \(K<\infty\), independent of \(\alpha\), such that for every \(0<\alpha\le\bar\alpha\),
\[
  \jsr(\mathcal A_\alpha)
  \le
  1-c\alpha
  <1,
\]
and
\begin{equation}\label{eq:fixed-data-product-bound}
\sup_{\mathbf M_0,\ldots,\mathbf M_{\ell-1}\in\co(\mathcal A_\alpha)}
\left\|
\mathbf M_{\ell-1}\cdots \mathbf M_0
\right\|_2
\le
K(1-c\alpha)^\ell,
\qquad
\forall \ell\ge0 .
\end{equation}
\end{lemma}

\begin{proof}
Let \(p\) and \(c\) be the norm and constant from \cref{lem:fixed-data-interpolation}.  Then, for every product generated by \(\co(\mathcal A_\alpha)\),
\[
  p(\mathbf M_{\ell-1}\cdots \mathbf M_0x)
  \le
  (1-c\alpha)^\ell p(x).
\]
Taking \(\ell\)-th roots, maximizing over products, and passing to the limit in the JSR definition gives
\[
  \jsr(\co(\mathcal A_\alpha))\le 1-c\alpha .
\]
By convex-hull invariance of the JSR in \cref{lem:lfa-convex-hull},
\[
  \jsr(\mathcal A_\alpha)=\jsr(\co(\mathcal A_\alpha))\le1-c\alpha .
\]
The construction in \cref{lem:fixed-data-interpolation} gives \(1-c\alpha>0\) for every \(0<\alpha\le\bar\alpha\), and hence the JSR bound is strict.

Since all norms on \(\R^m\) are equivalent, there are constants \(a,b<\infty\), independent of \(\alpha\), such that
\[
  \|x\|_2\le a p(x),
  \qquad
  p(x)\le b\|x\|_2,
  \qquad
  \forall x\in\R^m .
\]
Thus
\[
  \left\|
  \mathbf M_{\ell-1}\cdots \mathbf M_0x
  \right\|_2
  \le
  ab(1-c\alpha)^\ell\|x\|_2 .
\]
Setting \(K:=ab\) gives \cref{eq:fixed-data-product-bound}.
\end{proof}
The next lemma shows that the product-growth constant in~\cref{eq:iid-Kbeta-def} can be chosen uniformly when \(\beta_\varepsilon=1-c\alpha/2\).
\begin{lemma}\label{lem:fixed-data-Kbeta-uniform}
Under the assumptions of \cref{lem:fixed-data-uniform-product-bound}, let \(c>0\) and \(K<\infty\) be the constants in \cref{eq:fixed-data-product-bound}. Consequently, if
\[
  \beta_\varepsilon:=1-\frac{c}{2}\alpha,
\]
then \(\beta_\varepsilon>\jsr(\mathcal A_\alpha)\) and the product-growth constant in \cref{eq:iid-Kbeta-def} satisfies
\[
  K_{\beta_\varepsilon}
  \le
  K,
  \qquad
  0<\alpha\le\bar\alpha .
\]
\end{lemma}

\begin{proof}
By \cref{lem:fixed-data-uniform-product-bound},
\[
  \jsr(\mathcal A_\alpha)
  \le
  1-c\alpha
  <
  1-\frac{c}{2}\alpha
  =
  \beta_\varepsilon .
\]
Using \cref{eq:fixed-data-product-bound} in the definition of \(K_{\beta_\varepsilon}\) in \cref{eq:iid-Kbeta-def} gives
\[
\begin{aligned}
K_{\beta_\varepsilon}
&=
\sup_{\ell\ge0}
\sup_{\mathbf M_0,\ldots,\mathbf M_{\ell-1}\in\co(\mathcal A_\alpha)}
\beta_\varepsilon^{-\ell}
\left\|\mathbf M_{\ell-1}\cdots\mathbf M_0\right\|_2 \\
&\le
\sup_{\ell\ge0}
\beta_\varepsilon^{-\ell}K(1-c\alpha)^\ell \\
&=
K\sup_{\ell\ge0}
\left(
\frac{1-c\alpha}{\beta_\varepsilon}
\right)^\ell \\
&=
K\sup_{\ell\ge0}
\left(
\frac{1-c\alpha}{1-\frac{c}{2}\alpha}
\right)^\ell .
\end{aligned}
\]
The construction in \cref{lem:fixed-data-uniform-product-bound} gives \(0<1-c\alpha<1\) for \(0<\alpha\le\bar\alpha\). Therefore
\[
  0
  \le
  \frac{1-c\alpha}{1-\frac{c}{2}\alpha}
  <1,
\]
because
\[
  1-c\alpha
  <
  1-\frac{c}{2}\alpha
\]
for every \(\alpha>0\). Hence the geometric supremum is attained at \(\ell=0\), and
\[
K_{\beta_\varepsilon}
\le
K\sup_{\ell\ge0}
\left(
\frac{1-c\alpha}{1-\frac{c}{2}\alpha}
\right)^\ell
=K .
\]
\end{proof}
The next lemma connects the small-step discrete-time contraction to the associated continuous-time switched linear system.
\begin{lemma}\label{lem:fixed-data-continuous-time-stability}
Suppose that \cref{ass:fixed-data-alpha} holds and that there exists \(\bar\alpha>0\) such that
\[
  \jsr(\mathcal A_{\bar\alpha})<1.
\]
Moreover, for every piecewise-constant switching signal \(\sigma:[0,\infty)\to\Theta\) with finitely many switches on compact intervals, the continuous-time switched system
\[
  \dot x(t)=H_{\sigma(t)}x(t)
\]
is uniformly exponentially stable.
\end{lemma}

\begin{proof}
Let \(p\) and \(c\) be the norm and constant from \cref{lem:fixed-data-interpolation}. By the standard matrix-exponential limit formula~\cite[Chapter~10]{higham2008functions}, for each \(\tau>0\) and each policy \(\pi\),
\[
  e^{\tau H_\pi}
  =
  \lim_{n\to\infty}
  \left(I+\frac{\tau}{n}H_\pi\right)^n .
\]
For all sufficiently large \(n\), \(\tau/n\le\bar\alpha\) and \(1-c\tau/n>0\). Applying \cref{lem:fixed-data-interpolation} with \(\alpha=\tau/n\) gives
\[
  p\left(\left(I+\frac{\tau}{n}H_\pi\right)^n x\right)
  \le
  \left(1-c\frac{\tau}{n}\right)^n p(x).
\]
Since \(p\) is a norm, it is continuous.  Therefore,
\[
  p(e^{\tau H_\pi}x)
  =
  \lim_{n\to\infty}
  p\left(\left(I+\frac{\tau}{n}H_\pi\right)^n x\right).
\]
Taking limits in the preceding inequality yields
\[
\begin{aligned}
  p(e^{\tau H_\pi}x)
  &\le
  \lim_{n\to\infty}\left(1-c\frac{\tau}{n}\right)^n p(x) \\
  &=
  e^{-c\tau}p(x).
\end{aligned}
\]
This proves the single-mode estimate.  For a switching signal with modes \(\pi_1,\ldots,\pi_m\) and dwell times \(\tau_1,
\ldots,\tau_m\) on a compact interval, repeated application of the preceding bound gives
\[
  p\left(e^{\tau_mH_{\pi_m}}\cdots e^{\tau_1H_{\pi_1}}x\right)
  \le
  e^{-c(\tau_1+\cdots+\tau_m)}p(x).
\]
The vector inside \(p(\cdot)\) is the state reached at time \(\tau_1+\cdots+\tau_m\) by the continuous-time switched linear system that starts from \(x\), uses mode \(\pi_i\) for duration \(\tau_i\), and switches in the order \(\pi_1,\ldots,\pi_m\). Norm equivalence between \(p\) and \(\|\cdot\|_2\) gives constants \(C\ge1\) and \(\lambda>0\), independent of the switching signal, such that
\[
  \|x(t)\|_2\le C e^{-\lambda t}\|x(0)\|_2,
  \qquad t\ge0 .
\]
Thus the continuous-time switched system is uniformly exponentially stable.
\end{proof}

\begin{theorem}\label{thm:fixed-data-alpha-controlled-floor}
Assume that \cref{ass:fixed-data-alpha} holds, and suppose that \(\jsr(\mathcal A_{\bar\alpha})<1\) for some \(\bar\alpha>0\). For each \(0<\alpha\le\bar\alpha\), consider the i.i.d.\ linear Q-learning recursion \cref{eq:iid-lfa-update}. Assume that \cref{ass:iid-boundedness} holds with a common radius \(B_\theta\) that is independent of \(\alpha\). Let \(c>0\) and \(K<\infty\) be the constants from \cref{lem:fixed-data-uniform-product-bound}.  For each \(0<\alpha\le\bar\alpha\), choose
\[
  \beta_\varepsilon=1-\frac{c}{2}\alpha .
\]
Then \(\jsr(\mathcal A_\alpha)<1\), so there is a unique projected Bellman fixed point \(\theta^\star\) satisfying \cref{eq:pbe}.  For every \(0<\alpha\le\bar\alpha\) and every \(k\ge0\),
\begin{equation}\label{eq:fixed-data-finite-time-bound}
\begin{aligned}
\E[\|\theta_k-\theta^\star\|_2]
\le\;&
K\left(1-\frac{c}{2}\alpha\right)^k
\|\theta_0-\theta^\star\|_2 \\
&+
\alpha\gamma K^2 k
\left(1-\frac{c}{2}\alpha\right)^{k-1}
L\|\theta_0-\theta^\star\|_2 \\
&+
\left(
1+\frac{2\gamma K L}{c}
\right)
2K\phi_{\max}\sqrt{\frac{2\alpha}{c}} \\
&\qquad\times
\left(
R_{\max}+(1+\gamma)
\max_{(s,a)\in\calS\times\calA}
\|\phi(s,a)\|_1\,B_\theta
\right),
\end{aligned}
\end{equation}
with the convention that \(k(1-c\alpha/2)^{k-1}=0\) when \(k=0\).  The last term on the right-hand side is \(O(\sqrt{\alpha})\) as \(\alpha\downarrow0\).
\end{theorem}

\begin{proof}
By \cref{lem:fixed-data-uniform-product-bound}, \(\jsr(\mathcal A_\alpha)<1\) for every \(0<\alpha\le\bar\alpha\).  Hence \cref{thm:deterministic-jsr-lfa} gives the unique projected Bellman fixed point \(\theta^\star\) satisfying \cref{eq:pbe}.  By \cref{lem:fixed-data-Kbeta-uniform}, the choice \(\beta_\varepsilon=1-c\alpha/2\) satisfies
\[
  \beta_\varepsilon>\jsr(\mathcal A_\alpha),
  \qquad
  K_{\beta_\varepsilon}\le K .
\]
Moreover,
\begin{equation}\label{eq:fixed-data-beta-gap-proof}
  1-\beta_\varepsilon=\frac{c}{2}\alpha .
\end{equation}
Substituting \(K_{\beta_\varepsilon}\le K\) and \(\beta_\varepsilon=1-c\alpha/2\) into the first two terms of \cref{eq:iid-theta-bound} gives
\[
\begin{aligned}
&K\left(1-\frac{c}{2}\alpha\right)^k
\|\theta_0-\theta^\star\|_2 \\
&\quad+
\alpha\gamma K^2 k
\left(1-\frac{c}{2}\alpha\right)^{k-1}
L\|\theta_0-\theta^\star\|_2 .
\end{aligned}
\]
For the multiplicative factor in the stochastic term of \cref{eq:iid-theta-bound}, \cref{eq:fixed-data-beta-gap-proof} gives
\[
  1+
  \frac{\alpha\gamma K_{\beta_\varepsilon}}{1-\beta_\varepsilon}L
  \le
  1+
  \frac{\alpha\gamma K}{(c/2)\alpha}L
  =
  1+
  \frac{2\gamma K L}{c} .
\]
For the remaining stochastic factor, use
\[
\begin{aligned}
1-\beta_\varepsilon^2
&=
(1-\beta_\varepsilon)(1+\beta_\varepsilon) \\
&=
\frac{c}{2}\alpha(1+\beta_\varepsilon)
\ge
\frac{c}{2}\alpha .
\end{aligned}
\]
Therefore
\[
  \frac{2\alpha K_{\beta_\varepsilon}\phi_{\max}}
  {\sqrt{1-\beta_\varepsilon^2}}
  \le
  \frac{2\alpha K\phi_{\max}}
  {\sqrt{(c/2)\alpha}}
  =
  2K\phi_{\max}\sqrt{\frac{2\alpha}{c}} .
\]
Combining these three substitutions in~\cref{eq:iid-theta-bound} proves~\cref{eq:fixed-data-finite-time-bound}. Under~\cref{ass:fixed-data-alpha} and the common boundedness radius in the theorem statement, the constants \(\gamma\), \(\phi_{\max}\), \(R_{\max}\), \(B_\theta\), and \(L\) do not depend on \(\alpha\). Hence, the last term in~\cref{eq:fixed-data-finite-time-bound} is \(O(\sqrt{\alpha})\) as \(\alpha\downarrow0\).
\end{proof}

\subsection{Analysis in Small Step-Size Case}\label{subsec:projection-free-small-alpha-floor}
The bounded-iterate assumption in \cref{ass:iid-boundedness} is used in
\cref{thm:iid-jsr-lfa,thm:fixed-data-alpha-controlled-floor} only to make the
martingale-difference noise uniformly bounded. This
assumption can be removed for sufficiently small step-sizes. Throughout this subsection, we consider the same recursion~\cref{eq:iid-lfa-update}, but without imposing \cref{ass:iid-boundedness}. Let
\[
  x_k:=\theta_k-\theta^\star,
\]
where \(\theta^\star\) denotes the projected Bellman fixed point.
The next lemma replaces the uniform boundedness assumption in~\cref{ass:iid-boundedness} with a linear-growth bound on the martingale-difference noise.
\begin{lemma}\label{lem:iid-unprojected-noise-growth}
There exist constants \(\sigma_0,\sigma_1<\infty\), depending only
on the fixed problem data, such that the martingale-difference noise \(w_k\) in
\cref{eq:lfa-wk-def} satisfies
\[
  \E[w_k\mid\calF_k]=0
\]
and
\begin{equation}\label{eq:iid-unprojected-noise-growth}
  \E[\|w_k\|_2^2\mid\calF_k]
  \le
  \sigma_0^2+\sigma_1^2\|x_k\|_2^2,
  \qquad k\ge0 .
\end{equation}
\end{lemma}
\begin{proof}
For any \(\theta\in\R^m\), the sampled update vector satisfies
\[
\begin{aligned}
\|\widehat g_k(\theta)\|_2
&\le
\phi_{\max}
\left(
R_{\max}
+
\gamma\max_{u\in\calA}|\phi(s'_k,u)^\top\theta|
+
|\phi(s_k,a_k)^\top\theta|
\right) \\
&\le
\phi_{\max}
\left(
R_{\max}+(1+\gamma)\phi_{\max}\|\theta\|_2
\right).
\end{aligned}
\]
Substituting \(\theta=\theta^\star+x\) gives
\[
  \|\widehat g_k(\theta^\star+x)\|_2
  \le
  \phi_{\max}
  \left(
  R_{\max}+(1+\gamma)\phi_{\max}\|\theta^\star\|_2
  \right)
  +(1+\gamma)\phi_{\max}^2\|x\|_2 .
\]
Since \(g(\theta)=\E[\widehat g_k(\theta)\mid\calF_k]\), Jensen's inequality
implies the same bound for \(\|g(\theta)\|_2\).  Hence, for finite constants
\(W_0,W_1\) depending only on the fixed problem data,
\[
  \|w_k\|_2
  =
  \|\widehat g_k(\theta_k)-g(\theta_k)\|_2
  \le
  W_0+W_1\|x_k\|_2 .
\]
Consequently,
\[
\begin{aligned}
\|w_k\|_2^2
&\le
\left(W_0+W_1\|x_k\|_2\right)^2 \\
&\le
2W_0^2+2W_1^2\|x_k\|_2^2,
\end{aligned}
\]
where the second inequality uses \((a+b)^2\le2a^2+2b^2\).  Since \(x_k\) is \(\calF_k\)-measurable, taking conditional expectation gives
\[
\E[\|w_k\|_2^2\mid\calF_k]
\le
2W_0^2+2W_1^2\|x_k\|_2^2.
\]
Setting \(\sigma_0=\sqrt{2}W_0\) and \(\sigma_1=\sqrt{2}W_1\) proves \cref{eq:iid-unprojected-noise-growth}.  The martingale-difference identity follows from the definition of \(w_k\) and \cref{eq:mean-g-iid}.
\end{proof}
The following theorem combines the product bound in~\cref{lem:fixed-data-uniform-product-bound} with the linear-growth noise estimate in~\cref{lem:iid-unprojected-noise-growth} to obtain a finite-time bound without imposing~\cref{ass:iid-boundedness}.
\begin{theorem}\label{thm:projection-free-fixed-data-floor}
Assume \cref{ass:fixed-data-alpha}.  Suppose that there exists
\(\bar\alpha>0\) such that
\[
  \jsr(\mathcal A_{\bar\alpha})<1 .
\]
Let \(\theta^\star\) be the unique projected Bellman fixed point from
\cref{thm:deterministic-jsr-lfa}.  Let \(c>0\) and \(K<\infty\) be the constants from
\cref{lem:fixed-data-uniform-product-bound}.  Let \(\sigma_0,\sigma_1\) be the constants from
\cref{lem:iid-unprojected-noise-growth}, and let \(L\) be defined in
\cref{eq:iid-L-def}.  Define
\[
  C_0
  :=
  1+4K^2+4\left(\frac{\gamma L K^2}{c}\right)^2
\]
and
\[
  C_1
  :=
  4\left(
  \frac{K^2}{c}+
  \frac{\gamma^2L^2K^4}{c^3}
  \right).
\]
Set
\[
  \alpha_0:=
  \min\left\{
  \bar\alpha,
  \frac{1}{2C_1\sigma_1^2}
  \right\},
\]
with the convention that the second term is \(+\infty\) when \(\sigma_1=0\).
Consider the unprojected i.i.d.\ recursion \cref{eq:iid-lfa-update} with any deterministic initial condition
\(\theta_0\in\R^m\).  Then, for every \(0<\alpha\le\alpha_0\) and every \(k\ge0\),
\begin{equation}\label{eq:projection-free-ms-bounded}
\sup_{0\le t\le k}
\E[\|\theta_t-\theta^\star\|_2^2]
\le
2C_0\|\theta_0-\theta^\star\|_2^2
+
2C_1\alpha\sigma_0^2 .
\end{equation}
Consequently,
\[
\sup_{0\le t\le k}
\E[\|\theta_t-\theta^\star\|_2]
\le
\left(
2C_0\|\theta_0-\theta^\star\|_2^2
+
2C_1\alpha\sigma_0^2
\right)^{1/2}.
\]
\end{theorem}

\begin{proof}
By \cref{lem:fixed-data-uniform-product-bound}, for every
\(0<\alpha\le\bar\alpha\),
\begin{equation}\label{eq:projection-free-product-bound}
  \sup_{\mathbf M_0,\ldots,\mathbf M_{\ell-1}\in\co(\mathcal A_\alpha)}
  \left\|
  \mathbf M_{\ell-1}\cdots\mathbf M_0
  \right\|_2
  \le
  K(1-c\alpha)^\ell,
  \qquad \ell\ge0 .
\end{equation}
The same lemma implies \(0<1-c\alpha<1\) for
\(0<\alpha\le\bar\alpha\).  Since
\(\jsr(\mathcal A_\alpha)<1\) for such \(\alpha\),
\cref{thm:deterministic-jsr-lfa} gives the unique projected Bellman fixed point
\(\theta^\star\).  The exact direct representation gives, for each \(k\), an
\(\calF_k\)-measurable stochastic policy \(\mu_k\) such that
\begin{equation}\label{eq:projection-free-direct-error}
  x_{k+1}
  =
  \bA_{\mu_k}x_k+\alpha w_k,
  \qquad
  \E[w_k\mid\calF_k]=0 .
\end{equation}

Fix an arbitrary stochastic policy \(\bar\mu\).  Define the reference filter
\begin{equation}\label{eq:projection-free-reference-filter}
  y_{k+1}
  =
  \bA_{\bar\mu}y_k+\alpha w_k,
  \qquad
  y_0=x_0 .
\end{equation}
Write \(y_k=\bar y_k+\eta_k\), where
\[
  \bar y_{k+1}=\bA_{\bar\mu}\bar y_k,
  \qquad
  \bar y_0=x_0,
\]
and
\[
  \eta_{k+1}=\bA_{\bar\mu}\eta_k+\alpha w_k,
  \qquad
  \eta_0=0 .
\]
By \cref{eq:projection-free-product-bound},
\begin{equation}\label{eq:projection-free-ybar-bound}
  \|\bar y_k\|_2
  \le
  K(1-c\alpha)^k\|x_0\|_2 .
\end{equation}
Moreover,
\[
  \eta_k
  =
  \alpha\sum_{t=0}^{k-1}\bA_{\bar\mu}^{k-1-t}w_t .
\]
Expanding the square gives
\[
\begin{aligned}
\E[\|\eta_k\|_2^2]
&=
\alpha^2
\sum_{s=0}^{k-1}\sum_{t=0}^{k-1}
\E\!\left[
 w_s^\top
 (\bA_{\bar\mu}^{k-1-s})^\top
 \bA_{\bar\mu}^{k-1-t}w_t
\right].
\end{aligned}
\]
If \(s<t\), then
\(w_s^\top(\bA_{\bar\mu}^{k-1-s})^\top\bA_{\bar\mu}^{k-1-t}\) is
\(\calF_t\)-measurable. Hence
\[
\begin{aligned}
&\E\!\left[
 w_s^\top
 (\bA_{\bar\mu}^{k-1-s})^\top
 \bA_{\bar\mu}^{k-1-t}w_t
\right] \\
&\qquad =
\E\!\left[
 w_s^\top
 (\bA_{\bar\mu}^{k-1-s})^\top
 \bA_{\bar\mu}^{k-1-t}
 \E[w_t\mid\calF_t]
\right]
=0 .
\end{aligned}
\]
The case \(t<s\) is identical after conditioning on \(\calF_s\).  Therefore
only the diagonal terms remain, and the product bound in
\cref{eq:projection-free-product-bound} gives
\[
\begin{aligned}
\E[\|\eta_k\|_2^2]
&=
\alpha^2
\sum_{t=0}^{k-1}
\E\!\left[
\|\bA_{\bar\mu}^{k-1-t}w_t\|_2^2
\right] \\
&\le
\alpha^2K^2
\sum_{t=0}^{k-1}
(1-c\alpha)^{2(k-1-t)}
\E[\|w_t\|_2^2] \\
&\le
\alpha^2K^2
\sum_{t=0}^{k-1}
(1-c\alpha)^{2(k-1-t)}
\left(
\sigma_0^2+\sigma_1^2\E[\|x_t\|_2^2]
\right),
\end{aligned}
\]
where the last step uses \cref{lem:iid-unprojected-noise-growth} and the tower property.
For
\[
  M_k:=\sup_{0\le t\le k}\E[\|x_t\|_2^2],
\]
we have \(\E[\|x_t\|_2^2]\le M_{k-1}\) for \(0\le t\le k-1\). Thus, for
\(k\ge1\),
\[
\begin{aligned}
\E[\|\eta_k\|_2^2]
&\le
\alpha^2K^2
\left(
\sigma_0^2+\sigma_1^2M_{k-1}
\right)
\sum_{t=0}^{k-1}(1-c\alpha)^{2(k-1-t)} \\
&=
\alpha^2K^2
\left(
\sigma_0^2+\sigma_1^2M_{k-1}
\right)
\sum_{j=0}^{k-1}(1-c\alpha)^{2j} \\
&=
\alpha^2K^2
\left(
\sigma_0^2+\sigma_1^2M_{k-1}
\right)
\frac{1-(1-c\alpha)^{2k}}{1-(1-c\alpha)^2} \\
&\le
\frac{\alpha^2K^2}{1-(1-c\alpha)^2}
\left(
\sigma_0^2+\sigma_1^2M_{k-1}
\right).
\end{aligned}
\]
Since \(0<1-c\alpha<1\),
\[
  1-(1-c\alpha)^2
  =
  c\alpha(2-c\alpha)
  \ge
  c\alpha .
\]
Substituting this lower bound in the denominator gives, for every \(k\ge1\),
\begin{equation}\label{eq:projection-free-eta-bound}
  \E[\|\eta_k\|_2^2]
  \le
  \frac{\alpha K^2}{c}
  \left(
  \sigma_0^2+\sigma_1^2M_{k-1}
  \right).
\end{equation}

Next define the residual
\[
  r_k:=x_k-y_k .
\]
Subtracting \cref{eq:projection-free-reference-filter} from
\cref{eq:projection-free-direct-error} gives the noise-free recursion
\begin{equation}\label{eq:projection-free-residual-recursion}
  r_{k+1}
  =
  \bA_{\mu_k}r_k+
  (\bA_{\mu_k}-\bA_{\bar\mu})y_k,
  \qquad r_0=0 .
\end{equation}
Decompose \(r_k=u_k+v_k\), where
\[
  u_{k+1}
  =
  \bA_{\mu_k}u_k+
  (\bA_{\mu_k}-\bA_{\bar\mu})\bar y_k,
  \qquad u_0=0,
\]
and
\[
  v_{k+1}
  =
  \bA_{\mu_k}v_k+
  (\bA_{\mu_k}-\bA_{\bar\mu})\eta_k,
  \qquad v_0=0 .
\]
By \cref{eq:iid-mode-difference-bound},
\begin{equation}\label{eq:projection-free-mode-difference}
  \|\bA_\mu-\bA_\nu\|_2
  \le
  \alpha\gamma L,
  \qquad \forall \mu,\nu .
\end{equation}
Iterating the recursion for \(u_k\) gives, for \(k\ge1\),
\[
  u_k
  =
  \sum_{i=0}^{k-1}
  \bA_{\mu_{k-1}}\cdots\bA_{\mu_{i+1}}
  (\bA_{\mu_i}-\bA_{\bar\mu})\bar y_i,
\]
with the convention that the product is the identity when \(i=k-1\).  Hence
\cref{eq:projection-free-product-bound}, \cref{eq:projection-free-ybar-bound},
and \cref{eq:projection-free-mode-difference} imply
\[
\begin{aligned}
\|u_k\|_2
&\le
\sum_{i=0}^{k-1}
\left\|\bA_{\mu_{k-1}}\cdots\bA_{\mu_{i+1}}\right\|_2
\left\|\bA_{\mu_i}-\bA_{\bar\mu}\right\|_2
\|\bar y_i\|_2 \\
&\le
\sum_{i=0}^{k-1}
K(1-c\alpha)^{k-1-i}
\alpha\gamma L
K(1-c\alpha)^i
\|x_0\|_2 \\
&=
\alpha\gamma L K^2
\sum_{i=0}^{k-1}(1-c\alpha)^{k-1}
\|x_0\|_2 \\
&=
\alpha\gamma L K^2 k(1-c\alpha)^{k-1}\|x_0\|_2 .
\end{aligned}
\]
The scalar factor is bounded as
\[
\begin{aligned}
\alpha k(1-c\alpha)^{k-1}
&=
\frac{1}{c}\,c\alpha k(1-c\alpha)^{k-1} \\
&\le
\frac{1}{c}\,c\alpha
\sum_{j=0}^{k-1}(1-c\alpha)^j
\le
\frac{1}{c} .
\end{aligned}
\]
Therefore
\[
  \|u_k\|_2
  \le
  \frac{\gamma L K^2}{c}\|x_0\|_2 .
\]
Similarly, iterating the recursion for \(v_k\) gives
\[
  v_k
  =
  \sum_{i=0}^{k-1}
  \bA_{\mu_{k-1}}\cdots\bA_{\mu_{i+1}}
  (\bA_{\mu_i}-\bA_{\bar\mu})\eta_i .
\]
Using the same product and mode-difference bounds term by term,
\[
\begin{aligned}
  \|v_k\|_2
  &\le
  \sum_{i=0}^{k-1}
  K(1-c\alpha)^{k-1-i}
  \alpha\gamma L
  \|\eta_i\|_2 \\
  &=
  \alpha\gamma L K
  \sum_{i=0}^{k-1}
  (1-c\alpha)^{k-1-i}\|\eta_i\|_2 .
\end{aligned}
\]
Cauchy's inequality with weights \((1-c\alpha)^{k-1-i}\) yields
\[
\begin{aligned}
\E[\|v_k\|_2^2]
&\le
(\alpha\gamma L K)^2
\E\!\left[
\left(
\sum_{i=0}^{k-1}(1-c\alpha)^{k-1-i}\|\eta_i\|_2
\right)^2
\right] \\
&\le
(\alpha\gamma L K)^2
\left(\sum_{i=0}^{k-1}(1-c\alpha)^{k-1-i}\right)
\left(\sum_{i=0}^{k-1}(1-c\alpha)^{k-1-i}
      \E[\|\eta_i\|_2^2]\right) .
\end{aligned}
\]
Since
\[
  \sum_{i=0}^{k-1}(1-c\alpha)^{k-1-i}
  =
  \sum_{j=0}^{k-1}(1-c\alpha)^j
  \le
  \frac{1}{c\alpha},
\]
we get
\[
\begin{aligned}
\E[\|v_k\|_2^2]
&\le
(\alpha\gamma L K)^2
\frac{1}{c\alpha}
\frac{1}{c\alpha}
\sup_{0\le i\le k-1}\E[\|\eta_i\|_2^2] \\
&=
\frac{\gamma^2L^2K^2}{c^2}
\sup_{0\le i\le k-1}\E[\|\eta_i\|_2^2].
\end{aligned}
\]
Together with \cref{eq:projection-free-eta-bound}, this yields
\begin{equation}\label{eq:projection-free-v-bound}
  \E[\|v_k\|_2^2]
  \le
  \frac{\alpha\gamma^2L^2K^4}{c^3}
  \left(
  \sigma_0^2+\sigma_1^2M_{k-1}
  \right),
  \qquad k\ge1 .
\end{equation}

For any four vectors \(z_1,z_2,z_3,z_4\),
\[
  \|z_1+z_2+z_3+z_4\|_2^2
  \le
  \left(\sum_{j=1}^4\|z_j\|_2\right)^2
  \le
  4\sum_{j=1}^4\|z_j\|_2^2,
\]
where the last inequality is Cauchy's inequality.  Applying this with
\(z_1=\bar y_k\), \(z_2=\eta_k\), \(z_3=u_k\), and \(z_4=v_k\) gives
\[
  \|x_k\|_2^2
  \le
  4\|\bar y_k\|_2^2
  +4\|\eta_k\|_2^2
  +4\|u_k\|_2^2
  +4\|v_k\|_2^2 .
\]
For \(1\le t\le k\), the bounds above imply
\[
\begin{aligned}
\E[\|x_t\|_2^2]
&\le
4K^2\|x_0\|_2^2
+
4\frac{\alpha K^2}{c}
\left(
\sigma_0^2+\sigma_1^2M_{t-1}
\right) \\
&\quad+
4\left(\frac{\gamma L K^2}{c}\right)^2\|x_0\|_2^2
+
4\frac{\alpha\gamma^2L^2K^4}{c^3}
\left(
\sigma_0^2+\sigma_1^2M_{t-1}
\right) \\
&=
\left[
4K^2+4\left(\frac{\gamma L K^2}{c}\right)^2
\right]\|x_0\|_2^2
+
4\alpha
\left(
\frac{K^2}{c}+\frac{\gamma^2L^2K^4}{c^3}
\right)
\left(
\sigma_0^2+\sigma_1^2M_{t-1}
\right) \\
&\le
\left[
4K^2+4\left(\frac{\gamma L K^2}{c}\right)^2
\right]\|x_0\|_2^2
+
C_1\alpha
\left(
\sigma_0^2+\sigma_1^2M_{k-1}
\right),
\end{aligned}
\]
where \(M_{t-1}\le M_{k-1}\) and the definition of \(C_1\) was used. Since
\(M_0=\|x_0\|_2^2\), taking the supremum over \(0\le t\le k\) yields, for
\(k\ge1\),
\[
\begin{aligned}
M_k
&\le
\|x_0\|_2^2
+
\left[
4K^2+4\left(\frac{\gamma L K^2}{c}\right)^2
\right]\|x_0\|_2^2 \\
&\quad+
C_1\alpha
\left(
\sigma_0^2+\sigma_1^2M_{k-1}
\right) \\
&=
C_0\|x_0\|_2^2
+
C_1\alpha
\left(
\sigma_0^2+\sigma_1^2M_{k-1}
\right),
\end{aligned}
\]
where the last equality uses the definition of \(C_0\).
For \(0<\alpha\le\alpha_0\), the definition of \(\alpha_0\) gives
\[
  \alpha
  \le
  \frac{1}{2C_1\sigma_1^2}
\]
when \(\sigma_1>0\), and therefore
\[
  C_1\alpha\sigma_1^2
  \le
  C_1
  \frac{1}{2C_1\sigma_1^2}
  \sigma_1^2
  =
  \frac12 .
\]
When \(\sigma_1=0\), the same inequality is immediate because
\[
  C_1\alpha\sigma_1^2=0.
\]
Moreover, by definition,
\[
  M_{k-1}
  =
  \sup_{0\le t\le k-1}\E[\|x_t\|_2^2]
  \le
  \sup_{0\le t\le k}\E[\|x_t\|_2^2]
  =
  M_k .
\]
Thus, with \(q:=C_1\alpha\sigma_1^2\le1/2\), the preceding estimate implies
\[
\begin{aligned}
M_k
&\le
C_0\|x_0\|_2^2
+C_1\alpha\sigma_0^2
+qM_{k-1} \\
&\le
C_0\|x_0\|_2^2
+C_1\alpha\sigma_0^2
+qM_k .
\end{aligned}
\]
Therefore
\[
  (1-q)M_k
  \le
  C_0\|x_0\|_2^2
  +C_1\alpha\sigma_0^2 .
\]
Since \(1-q\ge1/2\),
\[
  M_k
  \le
  2C_0\|x_0\|_2^2
  +2C_1\alpha\sigma_0^2,
  \qquad k\ge1 .
\]
For \(k=0\), the same inequality holds because \(M_0=\|x_0\|_2^2\) and
\(C_0\ge1\). This proves \cref{eq:projection-free-ms-bounded}.  The first-moment bound follows from Jensen's inequality.
\end{proof}

The next result gives a projection-free finite-time error estimate under the fixed-data JSR condition of \cref{lem:fixed-data-uniform-product-bound} and the martingale-noise estimate of \cref{lem:iid-unprojected-noise-growth}.  It does not impose \cref{ass:iid-boundedness}.

\begin{theorem}\label{thm:iid-small-alpha-projection-free-error}
Assume \cref{ass:fixed-data-alpha}.  Suppose that there exists \(\bar\alpha>0\) such that
\[
  \jsr(\mathcal A_{\bar\alpha})<1 .
\]
Let \(c>0\) and \(K<\infty\) be the constants from \cref{lem:fixed-data-uniform-product-bound}.  Let \(\sigma_0,\sigma_1\) be the constants from \cref{lem:iid-unprojected-noise-growth}, and let \(L\) be defined in \cref{eq:iid-L-def}.  Consider the linear Q-learning recursion \cref{eq:iid-lfa-update} with a deterministic initial condition \(\theta_0\in\R^m\). For each \(0<\alpha\le\bar\alpha\), let \(\theta^\star\) be the unique projected Bellman fixed point from \cref{thm:deterministic-jsr-lfa}, and set
\[
  x_k:=\theta_k-\theta^\star .
\]
Define
\[
  C_0
  :=
  1+4K^2+4\left(\frac{\gamma L K^2}{c}\right)^2
\]
and
\[
  C_1
  :=
  4\left(
  \frac{K^2}{c}+
  \frac{\gamma^2L^2K^4}{c^3}
  \right).
\]
Set
\[
  \alpha_0
  :=
  \min\left\{
  \bar\alpha,
  \frac{1}{2C_1\sigma_1^2}
  \right\},
\]
where the second term is interpreted as \(+\infty\) when \(\sigma_1=0\).  Then, for every \(0<\alpha\le\alpha_0\) and every \(k\ge0\),
\begin{equation}\label{eq:iid-small-alpha-ms-bound}
\sup_{0\le t\le k}
\E[\|\theta_t-\theta^\star\|_2^2]
\le
2C_0\|\theta_0-\theta^\star\|_2^2
+
2C_1\alpha\sigma_0^2 .
\end{equation}
Moreover, with the convention that \(k(1-c\alpha/2)^{k-1}=0\) when \(k=0\),
\begin{equation}\label{eq:iid-small-alpha-finite-time-bound}
\begin{aligned}
\E[\|\theta_k-\theta^\star\|_2]
\le\;&
K\left(1-\frac{c}{2}\alpha\right)^k
\|\theta_0-\theta^\star\|_2 \\
&+
\alpha\gamma K^2 k
\left(1-\frac{c}{2}\alpha\right)^{k-1}
L\|\theta_0-\theta^\star\|_2 \\
&+
\left(
1+
\frac{2\gamma K L}{c}
\right)
K\sqrt{\frac{2\alpha}{c}} \\
&\qquad\times
\left(
\sigma_0^2+
\sigma_1^2
\left(
2C_0\|\theta_0-\theta^\star\|_2^2
+
2C_1\alpha\sigma_0^2
\right)
\right)^{1/2} .
\end{aligned}
\end{equation}
For fixed \(\theta_0\), the last term in \cref{eq:iid-small-alpha-finite-time-bound} is \(O(\sqrt{\alpha})\) as \(\alpha\downarrow0\).
\end{theorem}

\begin{proof}
The product estimate \cref{eq:fixed-data-product-bound} implies
\(\jsr(\mathcal A_\alpha)\le1-c\alpha<1\).  Hence
\cref{thm:deterministic-jsr-lfa} gives the unique projected Bellman fixed point
\(\theta^\star\).  The mean-square estimate \cref{eq:iid-small-alpha-ms-bound}
follows directly from \cref{thm:projection-free-fixed-data-floor}, applied with
the same constants \(c,K,\sigma_0,\sigma_1,L,C_0,C_1\) and the same admissible
step-size range.

By \cref{prop:iid-direct-lfa}, for each \(k\) there is an \(\calF_k\)-measurable
stochastic policy \(\mu_k\) such that
\begin{equation}\label{eq:iid-small-alpha-direct-error}
  x_{k+1}
  =
  \bA_{\mu_k}x_k+
  \alpha w_k .
\end{equation}
We now prove \cref{eq:iid-small-alpha-finite-time-bound}.  Define
\[
  \beta_\alpha:=1-\frac{c}{2}\alpha .
\]
Because \(1-c\alpha<\beta_\alpha<1\), \cref{eq:fixed-data-product-bound} implies that every product generated by \(\co(\mathcal A_\alpha)\) has Euclidean norm at most \(K\beta_\alpha^\ell\) at length \(\ell\).  Use the reference-filter decomposition from the proof of \cref{thm:projection-free-fixed-data-floor}: fix \(\bar\mu\), define \(y_{k+1}=\bA_{\bar\mu}y_k+\alpha w_k\) with \(y_0=x_0\), write \(y_k=\bar y_k+\eta_k\), and write \(x_k-y_k=u_k+v_k\) according to the two recursions driven by \(\bar y_k\) and \(\eta_k\).  The same estimates with this product bound give
\[
  \|\bar y_k\|_2
  \le
  K\beta_\alpha^k\|x_0\|_2
\]
and
\[
  \|u_k\|_2
  \le
  \alpha\gamma LK^2k\beta_\alpha^{k-1}\|x_0\|_2 .
\]
Let
\[
  M_\alpha:=
  2C_0\|x_0\|_2^2
  +
  2C_1\alpha\sigma_0^2 .
\]
By \cref{eq:iid-unprojected-noise-growth,eq:iid-small-alpha-ms-bound},
\[
  \E[\|w_t\|_2^2]
  \le
  \sigma_0^2+
  \sigma_1^2M_\alpha,
  \qquad 0\le t\le k .
\]
Therefore,
\[
\begin{aligned}
\E[\|\eta_k\|_2]
&\le
\left(
\alpha^2K^2
\sum_{t=0}^{k-1}
\beta_\alpha^{2(k-1-t)}
(\sigma_0^2+\sigma_1^2M_\alpha)
\right)^{1/2} \\
&\le
\frac{\alpha K}{\sqrt{1-\beta_\alpha^2}}
\left(
\sigma_0^2+
\sigma_1^2M_\alpha
\right)^{1/2} .
\end{aligned}
\]
Also,
\[
\begin{aligned}
\E[\|v_k\|_2]
&\le
\alpha\gamma L K
\sum_{i=0}^{k-1}
\beta_\alpha^{k-1-i}
\E[\|\eta_i\|_2] \\
&\le
\frac{\alpha\gamma L K}{1-\beta_\alpha}
\frac{\alpha K}{\sqrt{1-\beta_\alpha^2}}
\left(
\sigma_0^2+
\sigma_1^2M_\alpha
\right)^{1/2} .
\end{aligned}
\]
Since \(x_k=\bar y_k+\eta_k+u_k+v_k\), the preceding four estimates imply
\[
\begin{aligned}
\E[\|x_k\|_2]
\le\;&
K\beta_\alpha^k\|x_0\|_2
+
\alpha\gamma LK^2k\beta_\alpha^{k-1}\|x_0\|_2 \\
&+
\left(
1+
\frac{\alpha\gamma L K}{1-\beta_\alpha}
\right)
\frac{\alpha K}{\sqrt{1-\beta_\alpha^2}}
\left(
\sigma_0^2+
\sigma_1^2M_\alpha
\right)^{1/2} .
\end{aligned}
\]
Finally,
\[
  1-\beta_\alpha=\frac{c}{2}\alpha,
  \qquad
  1-\beta_\alpha^2
  =(1-\beta_\alpha)(1+\beta_\alpha)
  \ge
  \frac{c}{2}\alpha .
\]
Substituting these lower bounds and the definition of \(M_\alpha\) gives \cref{eq:iid-small-alpha-finite-time-bound}.  The final \(O(\sqrt{\alpha})\) statement follows because all constants in the last term are independent of \(\alpha\) and \(\theta_0\) is fixed.
\end{proof}

For simplicity and to make the main idea clear, this paper treats only i.i.d.\ observations in the stochastic convergence analysis.  Establishing convergence under Markovian observations is not the objective of this paper. The Markovian-observation case can be extended by following the direct-switching argument of~\cite{lee2026lyapunovcertified}, but we do not pursue it here.

\section{Regularized Linear Q-Learning}\label{sec:regularized-lfa}
This section adds the \(\ell_2\)-regularized term used in regularized Q-learning~\cite{limlee2024regularized}.  In the notation of this paper, the sampled regularization term is \(-\eta\theta_k\), where \(\eta\ge0\) is the regularization weight.  The deterministic regularized LFA map is
\begin{equation}\label{eq:regq-Talpha-def}
\bT_{\alpha,\eta}(\theta)
:=
\theta+
\alpha
\left[
\Phi^\top D
\left(R+\gamma PV_\theta-\Phi\theta\right)
-
\eta\theta
\right]
=
\bT_\alpha(\theta)-\alpha\eta\theta .
\end{equation}
The corresponding regularized projected Bellman fixed point, if it exists, is a parameter \(\theta_\eta^\star\) satisfying
\begin{equation}\label{eq:regq-pbe}
\Phi^\top D
\left(R+\gamma PV_{\theta_\eta^\star}-\Phi\theta_\eta^\star\right)
-
\eta\theta_\eta^\star
=0.
\end{equation}
Equivalently, with the definition of the regularized projection
\[
\boldsymbol{\Gamma}_\eta
:=
\Phi(\Phi^\top D\Phi+\eta I)^{-1}\Phi^\top D,
\]
\cref{eq:regq-pbe} can be written as the regularized projected Bellman equation
\begin{equation*}
\Phi\theta_\eta^\star
=
\boldsymbol{\Gamma}_\eta
\left(R+\gamma PV_{\theta_\eta^\star}\right).
\end{equation*}
Equivalently, because \(\alpha>0\), this fixed point satisfies the zero fixed-point-residual equation
\[
  \bT_{\alpha,\eta}(\theta_\eta^\star)-\theta_\eta^\star=0,
\]
or, equivalently, \(\bT_{\alpha,\eta}(\theta_\eta^\star)=\theta_\eta^\star\).
For \(\eta>0\), \(\boldsymbol{\Gamma}_\eta\) is generally not an idempotent \(D\)-orthogonal projection.  Thus the phrase regularized projected Bellman equation refers here to the regularized normal equation \cref{eq:regq-pbe}, rather than to an ordinary orthogonal projection.

The corresponding regularized projected Q-VI is
\begin{equation*}
\Phi\theta_{k+1}^{\mathrm{RPVI}}
=
\boldsymbol{\Gamma}_\eta
\left(R+\gamma PV_{\theta_k^{\mathrm{RPVI}}}\right),
\qquad k\in\{0,1,\ldots\},
\end{equation*}
or, equivalently, in parameter space,
\begin{equation}\label{eq:regq-rpvi-parameter}
\theta_{k+1}^{\mathrm{RPVI}}
=
(\Phi^\top D\Phi+\eta I)^{-1}\Phi^\top D
\left(R+\gamma PV_{\theta_k^{\mathrm{RPVI}}}\right),
\qquad k\in\{0,1,\ldots\}.
\end{equation}
Following the regularized Q-learning analysis of~\cite{limlee2024regularized}, a sufficient condition for the regularized projected Bellman operator
\[
Q\mapsto
\boldsymbol{\Gamma}_\eta\left(R+\gamma PV_Q\right)
\]
to be a contraction in the sup norm is
\begin{equation}\label{eq:regq-rpvi-contraction-condition}
\gamma\|\boldsymbol{\Gamma}_\eta P\|_\infty<1.
\end{equation}
Indeed, for any two representable Q-functions \(Q=\Phi\theta\) and \(\bar Q=\Phi\bar\theta\),
\[
\left\|
\boldsymbol{\Gamma}_\eta\gamma P\left(V_Q-V_{\bar Q}\right)
\right\|_\infty
\le
\gamma\|\boldsymbol{\Gamma}_\eta P\|_\infty\|Q-\bar Q\|_\infty.
\]
Because \(\boldsymbol{\Gamma}_\eta\to0\) as \(\eta\to\infty\), condition \cref{eq:regq-rpvi-contraction-condition} holds for all sufficiently large \(\eta\).  Hence, for sufficiently large \(\eta\), the regularized projected Bellman equation has a unique fixed point and the regularized projected Q-VI converges to it.

The corresponding regularized deterministic linear Q-learning iteration is
\begin{equation*}
\theta_{k+1}^{\mathrm{RegDLQ}}
=
\theta_k^{\mathrm{RegDLQ}}
+
\alpha
\left[
\Phi^\top D
\left(R+\gamma PV_{\theta_k^{\mathrm{RegDLQ}}}-\Phi\theta_k^{\mathrm{RegDLQ}}\right)
-
\eta\theta_k^{\mathrm{RegDLQ}}
\right],
\qquad k\in\{0,1,\ldots\}.
\end{equation*}
Equivalently,
\[
\theta_{k+1}^{\mathrm{RegDLQ}}
=
\bT_{\alpha,\eta}(\theta_k^{\mathrm{RegDLQ}}).
\]
The regularized deterministic linear Q-learning iteration can also be written as a residual step toward the regularized projected Q-VI update:
\begin{equation*}
\theta_{k+1}^{\mathrm{RegDLQ}}
=
\theta_k^{\mathrm{RegDLQ}}
+
\alpha(\Phi^\top D\Phi+\eta I)
\left(
\theta_{k+1}^{\mathrm{RPVI}}-\theta_k^{\mathrm{RegDLQ}}
\right),
\end{equation*}
where \(\theta_{k+1}^{\mathrm{RPVI}}\) is computed from \cref{eq:regq-rpvi-parameter} using \(\theta_k^{\mathrm{RPVI}}=\theta_k^{\mathrm{RegDLQ}}\).  Thus the two regularized iterations have the same fixed points when they converge, but their convergence behavior can be different because the deterministic linear Q-learning step contains the additional preconditioning factor \(\alpha(\Phi^\top D\Phi+\eta I)\).  When \(\eta=0\), the regularized map reduces to the unregularized linear Q-learning map in \cref{eq:Talpha-def}.  The regularized case therefore follows the same direct-switching template, with each mode shifted by the regularization term.

The next two examples show that the regularized projected Q-VI and the regularized deterministic linear Q-learning iteration can still have different convergence behavior.

\begin{example}\label{ex:reg-rpvi-converges-dlq-diverges}
Consider a one-state MDP with one action, zero reward, deterministic self-transition, $\gamma=0.9$, $\eta=1$, $\alpha=0.5$, $d(1,1)=1$, and the one-dimensional feature representation $\Phi=10$.
Since there is only one action, the maximization is trivial. In this case
\[
  \Phi^\top D\Phi=100,
  \qquad
  \Phi^\top DP\Phi=100.
\]
The regularized projected Q-VI satisfies
\[
  \theta_{k+1}^{\mathrm{RPVI}}
  =
  \gamma\frac{\Phi^\top DP\Phi}{\Phi^\top D\Phi+\eta}\theta_k^{\mathrm{RPVI}}
  =
  \frac{90}{101}\theta_k^{\mathrm{RPVI}},
  \qquad k\in\{0,1,\ldots\},
\]
which converges to \(0\). However, the regularized deterministic linear Q-learning iteration satisfies
\[
\begin{aligned}
  \theta_{k+1}^{\mathrm{RegDLQ}}
  &=
  \left(
  1-
  \alpha(\Phi^\top D\Phi+\eta)
  +
  \alpha\gamma\Phi^\top DP\Phi
  \right)\theta_k^{\mathrm{RegDLQ}} \\
  &=
  \left(1-0.5\cdot101+0.5\cdot90\right)
  \theta_k^{\mathrm{RegDLQ}}
  =
  -4.5\theta_k^{\mathrm{RegDLQ}},
  \qquad k\in\{0,1,\ldots\}.
\end{aligned}
\]
Thus regularized deterministic linear Q-learning diverges for every nonzero initial condition, even though the regularized projected Q-VI converges to the same fixed point.
\end{example}

\begin{example}\label{ex:reg-dlq-converges-rpvi-diverges}
Consider a two-state MDP with one action, zero reward, discount factor $\gamma=0.9$, $\eta=1$, $\alpha=0.1$, and transition matrix
\[
  P=
  \begin{bmatrix}
  0 & 1\\
  0 & 1
  \end{bmatrix}.
\]
Let the sampling distribution and feature matrix be
\[
  d(1,1)=0.99,
  \qquad
  d(2,1)=0.01,
  \qquad
  \Phi=
  \begin{bmatrix}
  1\\ -10
  \end{bmatrix}.
\]
Again, the maximization is trivial.  As in the corresponding unregularized example,
\[
  \Phi^\top D\Phi=1.99,
  \qquad
  \Phi^\top DP\Phi=-8.9.
\]
The regularized projected Q-VI is
\[
  \theta_{k+1}^{\mathrm{RPVI}}
  =
  \gamma\frac{\Phi^\top DP\Phi}{\Phi^\top D\Phi+\eta}\theta_k^{\mathrm{RPVI}}
  =
  -\frac{8.01}{2.99}\theta_k^{\mathrm{RPVI}},
  \qquad k\in\{0,1,\ldots\}.
\]
Since \(8.01/2.99>1\), the regularized projected Q-VI diverges for every nonzero initial condition.  On the other hand, the regularized deterministic linear Q-learning iteration satisfies
\[
\begin{aligned}
  \theta_{k+1}^{\mathrm{RegDLQ}}
  &=
  \left(
  1-
  \alpha(\Phi^\top D\Phi+\eta)
  +
  \alpha\gamma\Phi^\top DP\Phi
  \right)\theta_k^{\mathrm{RegDLQ}} \\
  &=
  \left(1-0.1\cdot2.99+0.1\cdot0.9\cdot(-8.9)\right)
  \theta_k^{\mathrm{RegDLQ}}
  =
  -0.1\theta_k^{\mathrm{RegDLQ}},
  \qquad k\in\{0,1,\ldots\}.
\end{aligned}
\]
Therefore regularized deterministic linear Q-learning converges to \(0\), while the regularized projected Q-VI diverges.
\end{example}

\subsection{Regularized Direct Modes and JSR}
For a stochastic policy \(\mu\), define the regularized direct mode
\[
\bA_\mu^\eta
:=
I-
\alpha(\Phi^\top D\Phi+\eta I)
+
\alpha\gamma\Phi^\top DP\bPi^\mu\Phi
=
\bA_\mu-\alpha\eta I .
\]
For deterministic policies, define
\[
\mathcal A_{\alpha,\eta}
:=
\left\{
\bA_\pi^\eta
:=
I-
\alpha(\Phi^\top D\Phi+\eta I)
+
\alpha\gamma\Phi^\top DP\bPi^\pi\Phi:
\pi\in\Theta
\right\}.
\]
The regularized direct JSR rate is denoted by
\[
\rho_{\alpha,\eta}
:=
\jsr(\mathcal A_{\alpha,\eta}).
\]

The next proposition is the regularized counterpart of the pairwise direct representation.

\begin{proposition}\label{prop:regq-pairwise-direct-lfa}
For every \(\theta,\bar\theta\in\R^m\), there exists a stochastic policy \(\mu_{\theta,\bar\theta}\) such that
\[
\bT_{\alpha,\eta}(\theta)-\bT_{\alpha,\eta}(\bar\theta)
=
\bA_{\mu_{\theta,\bar\theta}}^{\eta}(\theta-\bar\theta).
\]
In particular, if \(\theta_\eta^\star\) satisfies \cref{eq:regq-pbe} and \(x_k:=\theta_k-\theta_\eta^\star\), then the deterministic regularized recursion \(\theta_{k+1}=\bT_{\alpha,\eta}(\theta_k)\) satisfies
\begin{equation}\label{eq:regq-det-direct-error}
x_{k+1}=\bA_{\mu_k}^{\eta}x_k,
\qquad k\in\{0,1,\ldots\},
\end{equation}
where \(\mu_k\) is a stochastic policy depending measurably on \(\theta_k\) and \(\theta_\eta^\star\).
\end{proposition}

\begin{proof}
Using \cref{eq:regq-Talpha-def},
\[
\bT_{\alpha,\eta}(\theta)-\bT_{\alpha,\eta}(\bar\theta)
=
\bT_\alpha(\theta)-\bT_\alpha(\bar\theta)-\alpha\eta(\theta-\bar\theta).
\]
By \cref{prop:pairwise-direct-lfa}, there exists a stochastic policy \(\mu_{\theta,\bar\theta}\) such that
\[
\bT_\alpha(\theta)-\bT_\alpha(\bar\theta)
=
\bA_{\mu_{\theta,\bar\theta}}(\theta-\bar\theta).
\]
Therefore
\[
\bT_{\alpha,\eta}(\theta)-\bT_{\alpha,\eta}(\bar\theta)
=
(\bA_{\mu_{\theta,\bar\theta}}-\alpha\eta I)(\theta-\bar\theta)
=
\bA_{\mu_{\theta,\bar\theta}}^{\eta}(\theta-\bar\theta).
\]
If \(\theta_\eta^\star\) is a fixed point of \(\bT_{\alpha,\eta}\) and \(\bar\theta=\theta_\eta^\star\), then \cref{eq:regq-det-direct-error} follows.
\end{proof}

\begin{lemma}\label{lem:regq-convex-hull}
For every stochastic policy \(\mu\),
\[
\bA_\mu^\eta\in\co(\mathcal A_{\alpha,\eta}).
\]
Moreover,
\[
\jsr(\co(\mathcal A_{\alpha,\eta}))
=
\jsr(\mathcal A_{\alpha,\eta}).
\]
\end{lemma}

\begin{proof}
For a stochastic policy \(\mu\), the standard convex-hull identity for stationary policies gives convex weights \(c_\pi(\mu)\) satisfying
\[
\bPi^\mu=\sum_{\pi\in\Theta}c_\pi(\mu)\bPi^\pi,
\qquad
c_\pi(\mu)\ge0,
\qquad
\sum_{\pi\in\Theta}c_\pi(\mu)=1.
\]
Using the affine dependence of \(\bA_\mu^\eta\) on \(\bPi^\mu\),
\[
\begin{aligned}
\bA_\mu^\eta
&=
I-
\alpha(\Phi^\top D\Phi+\eta I)
+
\alpha\gamma\Phi^\top DP\bPi^\mu\Phi\\
&=
\sum_{\pi\in\Theta}c_\pi(\mu)
\left(
I-
\alpha(\Phi^\top D\Phi+\eta I)
+
\alpha\gamma\Phi^\top DP\bPi^\pi\Phi
\right)\\
&=
\sum_{\pi\in\Theta}c_\pi(\mu)\bA_\pi^\eta.
\end{aligned}
\]
Thus \(\bA_\mu^\eta\in\co(\mathcal A_{\alpha,\eta})\).  The equality of the JSR before and after convexification is the standard convex-hull invariance of the JSR for finite matrix families, applied to \(\mathcal A_{\alpha,\eta}\).
\end{proof}

The regularized convergence analysis can be obtained by applying almost the same interpretation as in the preceding linear Q-learning analysis to the shifted family \(\mathcal A_{\alpha,\eta}\). In particular, when \(\jsr(\mathcal A_{\alpha,\eta})<1\), the JSR Lyapunov construction applied to the modes \(\bA_\pi^\eta\) gives a contraction of the deterministic regularized map around the regularized projected Bellman fixed point \(\theta_\eta^\star\). The sampled i.i.d.\ regularized recursion is handled in the same way under the same bounded-iterate assumption used in \cref{thm:iid-jsr-lfa}: the conditional mean drift uses the shifted regularized modes, while the sampling noise is the same martingale-difference term as in the unregularized linear Q-learning case. Therefore, by replacing \(\bA_\pi\) with \(\bA_\pi^\eta\) and \(\theta^\star\) with \(\theta_\eta^\star\) in the previous linear Q-learning argument, the corresponding deterministic convergence and stochastic bounded-error bounds follow. Since these derivations overlap almost completely with the earlier linear Q-learning analysis, we omit the details.

The regularization replaces each drift mode \(\bA_\pi\) by \(\bA_\pi^\eta=\bA_\pi-\alpha\eta I\). Therefore the convergence certificate is obtained by computing or bounding the JSR of the shifted family \(\mathcal A_{\alpha,\eta}\).

\subsection{Regularization-Dependent JSR Upper Bounds}
This subsection presents three regularization-dependent consequences of the shifted mode representation.  First, the regularized JSR can be written as a rescaled unregularized JSR\@. Second, when \(0\le\alpha\eta\le1\), the same representation gives a Euclidean upper bound that shows the interaction between \(\alpha\) and \(\eta\). Third, a more conservative all-\(\eta\) bound follows by bounding the shifted drift norm directly.

Let
\begin{equation}\label{eq:cL-def}
\begin{aligned}
c_\Phi
&:=
\min_{\pi\in\Theta}
\lambda_{\min}\left(
\frac{
\Phi^\top D\Phi-
\gamma\Phi^\top DP\bPi^\pi\Phi
+
\left(\Phi^\top D\Phi-
\gamma\Phi^\top DP\bPi^\pi\Phi\right)^\top
}{2}
\right),\\
L_\Phi
&:=
\max_{\pi\in\Theta}
\left\|
\Phi^\top D\Phi-
\gamma\Phi^\top DP\bPi^\pi\Phi
\right\|_2 .
\end{aligned}
\end{equation}
For each deterministic policy \(\pi\), let us define
\[
\bA_\pi^\eta
=
I-
\alpha\left(
\Phi^\top D\Phi-
\gamma\Phi^\top DP\bPi^\pi\Phi
+
\eta I
\right).
\]
Thus, the regularized switching system matrix is a scalar shift of each unregularized switching system matrix. In the regularized Q-learning analysis of~\cite{limlee2024regularized}, increasing \(\eta\) strengthens the regularized stability mechanism. In the present fixed-\(\alpha\) switched-system analysis, however, this monotone interpretation does not directly apply: increasing \(\eta\) also changes the discrete-time factor \(1-\alpha\eta\) and the effective step-size scaling made precise below.
\begin{proposition}\label{prop:regq-jsr-rescaling}
For \(\bar\alpha\in\R\), define the formal switching family
\[
\mathcal B_{\bar\alpha}
:=
\left\{
I-
\bar\alpha\left(\Phi^\top D\Phi-
\gamma\Phi^\top DP\bPi^\pi\Phi\right):
\pi\in\Theta
\right\}.
\]
If \(\alpha\eta\ne1\), then
\begin{equation}\label{eq:regq-jsr-rescaling}
\jsr(\mathcal A_{\alpha,\eta})
=
|1-\alpha\eta|\,
\jsr\left(\mathcal B_{\alpha/(1-\alpha\eta)}\right).
\end{equation}
In particular, if \(0\le \alpha\eta<1\), then
\[
\jsr(\mathcal A_{\alpha,\eta})
=
(1-\alpha\eta)\,
\jsr\left(\mathcal B_{\alpha/(1-\alpha\eta)}\right).
\]
If \(\alpha\eta=1\), then
\begin{equation}\label{eq:regq-jsr-rescaling-critical}
\jsr(\mathcal A_{\alpha,\eta})
=
\alpha\,
\jsr\left(
\left\{
\Phi^\top D\Phi-
\gamma\Phi^\top DP\bPi^\pi\Phi:
\pi\in\Theta
\right\}
\right).
\end{equation}
\end{proposition}

\begin{proof}
For each deterministic policy \(\pi\),
\[
\bA_\pi^\eta
=
(1-\alpha\eta)I-
\alpha\left(\Phi^\top D\Phi-
\gamma\Phi^\top DP\bPi^\pi\Phi\right).
\]
If \(\alpha\eta\ne1\), then
\[
\bA_\pi^\eta
=
(1-\alpha\eta)
\left(
I-
\frac{\alpha}{1-\alpha\eta}
\left(\Phi^\top D\Phi-
\gamma\Phi^\top DP\bPi^\pi\Phi\right)
\right).
\]
Every length-\(k\) product of matrices in \(\mathcal A_{\alpha,\eta}\) is therefore \((1-\alpha\eta)^k\) times the corresponding length-\(k\) product generated by \(\mathcal B_{\alpha/(1-\alpha\eta)}\).  Taking the supremum over products, the \(k\)th root, and the limit gives \cref{eq:regq-jsr-rescaling}.  If \(\alpha\eta=1\), then
\[
\bA_\pi^\eta
=
-\alpha\left(\Phi^\top D\Phi-
\gamma\Phi^\top DP\bPi^\pi\Phi\right),
\]
and homogeneity of the JSR gives \cref{eq:regq-jsr-rescaling-critical}.
\end{proof}

\begin{proposition}\label{prop:regq-euclidean-eta-jsr}
Let \(c_\Phi\) and \(L_\Phi\) be defined in \cref{eq:cL-def}.  If \(0\le \alpha\eta\le1\), then
\begin{equation}\label{eq:regq-eta-jsr-bound-factorized}
\jsr(\mathcal A_{\alpha,\eta})
\le
\sqrt{
(1-\alpha\eta)^2
-2\alpha(1-\alpha\eta)c_\Phi
+\alpha^2L_\Phi^2
}.
\end{equation}
Equivalently,
\begin{equation}\label{eq:regq-eta-jsr-bound-expanded}
\jsr(\mathcal A_{\alpha,\eta})
\le
\sqrt{
1
-2\alpha(c_\Phi+\eta)
+\alpha^2
\left(L_\Phi^2+2c_\Phi\eta+\eta^2\right)
}.
\end{equation}
Consequently, if
\begin{equation}\label{eq:regq-eta-alpha-condition}
0\le\alpha\eta\le1,
\qquad
c_\Phi+\eta>0,
\qquad
0<\alpha<
\frac{2(c_\Phi+\eta)}{L_\Phi^2+2c_\Phi\eta+\eta^2},
\end{equation}
then \(\jsr(\mathcal A_{\alpha,\eta})<1\).

\end{proposition}

\begin{proof}
For each deterministic policy \(\pi\),
\[
\bA_\pi^\eta
=
(1-\alpha\eta)I-
\alpha\left(\Phi^\top D\Phi-
\gamma\Phi^\top DP\bPi^\pi\Phi\right).
\]
For any \(x\in\R^m\), using \(0\le\alpha\eta\le1\),
\[
\begin{aligned}
\|\bA_\pi^\eta x\|_2^2
&=
\left\|
(1-\alpha\eta)x-
\alpha\left(\Phi^\top D\Phi-
\gamma\Phi^\top DP\bPi^\pi\Phi\right)x
\right\|_2^2\\
&=
(1-\alpha\eta)^2\|x\|_2^2
-2\alpha(1-\alpha\eta)x^\top
\left(\Phi^\top D\Phi-
\gamma\Phi^\top DP\bPi^\pi\Phi\right)x\\
&\qquad
+
\alpha^2\left\|
\left(\Phi^\top D\Phi-
\gamma\Phi^\top DP\bPi^\pi\Phi\right)x
\right\|_2^2\\
&\le
\left[
(1-\alpha\eta)^2
-2\alpha(1-\alpha\eta)c_\Phi
+\alpha^2L_\Phi^2
\right]\|x\|_2^2.
\end{aligned}
\]
Thus
\[
\max_{\pi\in\Theta}\|\bA_\pi^\eta\|_2
\le
\sqrt{
(1-\alpha\eta)^2
-2\alpha(1-\alpha\eta)c_\Phi
+\alpha^2L_\Phi^2
}.
\]
Since the JSR is bounded above by any common induced-norm bound, \cref{eq:regq-eta-jsr-bound-factorized} follows.  Expanding the right-hand side gives \cref{eq:regq-eta-jsr-bound-expanded}.  The condition \cref{eq:regq-eta-alpha-condition} is exactly the condition that the squared bound in \cref{eq:regq-eta-jsr-bound-expanded} is strictly smaller than one.
\end{proof}

\begin{corollary}\label{cor:regq-all-eta-bound}
For every \(\eta\ge0\), define
\[
L_{\Phi,\eta}
:=
\max_{\pi\in\Theta}
\left\|
\Phi^\top D\Phi-
\gamma\Phi^\top DP\bPi^\pi\Phi
+
\eta I
\right\|_2.
\]
Then
\begin{equation}\label{eq:regq-all-eta-bound-sharp}
\jsr(\mathcal A_{\alpha,\eta})
\le
\sqrt{
1-2\alpha(c_\Phi+\eta)
+\alpha^2L_{\Phi,\eta}^2
}.
\end{equation}
Since \(L_{\Phi,\eta}\le L_\Phi+\eta\), the more conservative but directly computable bound
\begin{equation}\label{eq:regq-all-eta-bound-conservative}
\jsr(\mathcal A_{\alpha,\eta})
\le
\sqrt{
1-2\alpha(c_\Phi+\eta)
+\alpha^2(L_\Phi+\eta)^2
}
\end{equation}
also holds.  Hence, if
\begin{equation}\label{eq:regq-all-eta-alpha-condition}
c_\Phi+\eta>0,
\qquad
0<\alpha<
\frac{2(c_\Phi+\eta)}{L_{\Phi,\eta}^2},
\end{equation}
then \(\jsr(\mathcal A_{\alpha,\eta})<1\).  A sufficient condition using only \(L_\Phi\) is obtained by replacing \(L_{\Phi,\eta}^2\) with \((L_\Phi+\eta)^2\) in \cref{eq:regq-all-eta-alpha-condition}.
\end{corollary}

\begin{proof}
For each deterministic policy \(\pi\),
\[
\bA_\pi^\eta
=
I-
\alpha\left(
\Phi^\top D\Phi-
\gamma\Phi^\top DP\bPi^\pi\Phi
+
\eta I
\right).
\]
By the definition of \(c_\Phi\),
\[
\lambda_{\min}\left(
\frac{
\Phi^\top D\Phi-
\gamma\Phi^\top DP\bPi^\pi\Phi
+
\left(\Phi^\top D\Phi-
\gamma\Phi^\top DP\bPi^\pi\Phi\right)^\top
}{2}
+
\eta I
\right)
\ge
c_\Phi+
\eta,
\]
and, by definition, the corresponding norm is bounded by \(L_{\Phi,\eta}\), with \(L_{\Phi,\eta}\le L_\Phi+\eta\).  Therefore, for every \(x\in\R^m\),
\[
\begin{aligned}
\|\bA_\pi^\eta x\|_2^2
&=
\left\|x-
\alpha\left(
\Phi^\top D\Phi-
\gamma\Phi^\top DP\bPi^\pi\Phi
+
\eta I
\right)x\right\|_2^2\\
&\le
\left[
1-2\alpha(c_\Phi+\eta)+\alpha^2L_{\Phi,\eta}^2
\right]\|x\|_2^2.
\end{aligned}
\]
Taking the maximum over \(\pi\) and using the induced-norm upper bound on the JSR gives \cref{eq:regq-all-eta-bound-sharp}.  Replacing \(L_{\Phi,\eta}\) by \(L_\Phi+\eta\) gives \cref{eq:regq-all-eta-bound-conservative}.  The step-size condition \cref{eq:regq-all-eta-alpha-condition} makes the squared bound strictly smaller than one.
\end{proof}

\begin{example}\label{ex:regq-eta-stabilizes-jsr}
Consider a two-state MDP with one action, zero reward, discount factor $\gamma=0.9$, $\alpha=0.1$, and transition matrix
\[
  P=
  \begin{bmatrix}
  0 & 1\\
  0 & 1
  \end{bmatrix}.
\]
Let
\[
  d(1,1)=0.9,
  \qquad
  d(2,1)=0.1,
  \qquad
  \Phi=
  \begin{bmatrix}
  1\\ 10
  \end{bmatrix}.
\]
Since there is only one action, the maximization is trivial and there is only one direct mode.  We have
\[
\Phi^\top D\Phi
=0.9\cdot1^2+0.1\cdot10^2
=10.9,
\]
and, since \(P\Phi=(10,10)^\top\),
\[
\Phi^\top DP\Phi
=0.9\cdot1\cdot10+0.1\cdot10\cdot10
=19.
\]
Hence
\[
\Phi^\top D\Phi-\gamma\Phi^\top DP\Phi
=10.9-0.9\cdot19
=-6.2.
\]
The unregularized direct mode is
\[
\bA
=I-\alpha(\Phi^\top D\Phi-\gamma\Phi^\top DP\Phi)
=1-0.1(-6.2)
=1.62,
\]
and therefore
\[
  \jsr(\mathcal A_\alpha)=1.62>1.
\]
Now choose $\eta=20$. Then \(\alpha\eta=2\), so the restriction \(0\le\alpha\eta\le1\) in \cref{prop:regq-euclidean-eta-jsr} is not available.  However, \cref{cor:regq-all-eta-bound} applies.  In this scalar example,
\[
  c_\Phi=-6.2,
  \qquad
  L_\Phi=6.2,
  \qquad
  L_{\Phi,\eta}=|\Phi^\top D\Phi-\gamma\Phi^\top DP\Phi+\eta|=13.8.
\]
The step-size condition in \cref{cor:regq-all-eta-bound} holds because
\[
  c_\Phi+\eta=13.8>0,
  \qquad
  0<\alpha=0.1
  <
  \frac{2(c_\Phi+\eta)}{L_{\Phi,\eta}^2}
  =
  \frac{27.6}{13.8^2}
  =
  \frac{2}{13.8}.
\]
Therefore, the sharp bound in \cref{eq:regq-all-eta-bound-sharp} gives
\[
\jsr(\mathcal A_{\alpha,\eta})
\le
\sqrt{
1-2\cdot0.1\cdot13.8+0.1^2\cdot13.8^2
}
=0.38<1.
\]
Directly, the regularized direct mode is
\[
\bA^\eta
=I-\alpha(\Phi^\top D\Phi-\gamma\Phi^\top DP\Phi+\eta)
=1-0.1(13.8)
=-0.38,
\]
so
\[
  \jsr(\mathcal A_{\alpha,\eta})=0.38<1.
\]
Thus, although the unregularized direct JSR is larger than one, \cref{cor:regq-all-eta-bound} certifies that this regularized direct JSR is strictly smaller than one.
\end{example}

The bounds above should therefore not be read as saying that \(\jsr(\mathcal A_{\alpha,\eta})\) is monotone decreasing in \(\eta\). The exact identity \cref{eq:regq-jsr-rescaling} shows that increasing \(\eta\) simultaneously introduces the scalar factor \(|1-\alpha\eta|\) and changes the effective step size to \(\alpha/(1-\alpha\eta)\) when \(\alpha\eta\ne1\). For this reason, regularization may improve a stability certificate by increasing the accretivity term \(c_\Phi+\eta\), but with a fixed scalar step size \(\alpha\), larger \(\eta\) can also make the discrete-time step-size restriction more severe. This is the main distinction from analyses in which larger regularization directly strengthens stability after the algorithmic scaling is adjusted accordingly.

\section{Conclusion}\label{sec:conclusion}
This paper introduced an SLS framework for linear Q-learning with LFA\@. Starting from a stochastic SLS representation of the linear Q-learning error, we derived JSR-based finite-time bounds and convergence certificates whose leading exponential rate is tied to the intrinsic worst-case rate of the induced switching family. This product-level certificate distinguishes the analysis from one-step norm bounds and other scalar stability arguments. The same viewpoint also gives a unified parameter-space interpretation of regularized Q-learning with LFA, connecting projected Bellman equations, stochastic-policy switching, and switched-system stability.

\appendix

\section{Proofs for the Linear Switching System Representation}\label{app:direct-lfa-proofs}

\subsection{Trajectory Example for the Linear Switching System Representation}\label{app:det-qlearning-switching-trajectory-example}
\begin{example}\label{ex:det-qlearning-switching-trajectory}
Consider the one-state two-action MDP, $\calS=\{1\},\calA=\{1,2\}$, with zero reward and deterministic self-transition. Let $\gamma=\frac12,\alpha=0.9, d(1,1)=0.9,d(1,2)=0.1$, and use the one-dimensional feature representation
\[
  \phi(1,1)=1,
  \qquad
  \phi(1,2)=-2,
  \qquad
  \Phi=\begin{bmatrix}1\\ -2\end{bmatrix}.
\]
For this example,
\[
  \Phi^\top D\Phi=0.9\cdot 1^2+0.1\cdot (-2)^2=1.3,
  \qquad
  \Phi^\top DP=0.9\cdot1+0.1\cdot(-2)=0.7.
\]
Since the reward is zero and the next state is always the single state,
\[
  V_\theta=\max\{\theta,-2\theta\},
\]
and the deterministic linear Q-learning map is
\[
\begin{aligned}
  \bT_\alpha(\theta) =\theta+\alpha\Phi^\top D\left(\gamma PV_\theta-\Phi\theta\right) =-0.17\theta+0.315\max\{\theta,-2\theta\}.
\end{aligned}
\]
Equivalently,
\[
  \bT_\alpha(\theta)
  =
  \begin{cases}
    0.145\theta, & \theta\ge0,\\
    -0.8\theta, & \theta<0.
  \end{cases}
\]
Consequently, \(\theta^\star=0\) is a projected Bellman fixed point.  The two deterministic-policy direct modes are
\[
  \bA_1
  =1-0.9\cdot1.3+0.9\cdot\frac12\cdot0.7\cdot1
  =0.145,
  \qquad
  \bA_2
  =1-0.9\cdot1.3+0.9\cdot\frac12\cdot0.7\cdot(-2)
  =-0.8.
\]
Now choose the initial parameter \(\theta_0=-2\).  The deterministic Q-learning trajectory is
\[
  \theta_1=\bT_\alpha(-2)=1.6,
  \qquad
  \theta_2=\bT_\alpha(1.6)=0.232,
  \qquad
  \theta_3=\bT_\alpha(0.232)=0.03364,
\]
and, in closed form,
\[
  \theta_k=1.6(0.145)^{k-1},
  \qquad k\ge1.
\]
The greedy action is action 2 at \(k=0\), because \(-2\theta_0>\theta_0\), and action 1 for every \(k\ge1\), because \(\theta_k>0\).  Hence the corresponding switching mode sequence is
\[
  \sigma_0=2,
  \qquad
  \sigma_k=1\quad(k\ge1).
\]
For the switching system state \(x_k:=\theta_k-\theta^\star=\theta_k\), the switched linear trajectory satisfies
\[
  x_{k+1}=\bA_{\sigma_k}x_k,
  \qquad
  x_0=-2.
\]
Therefore
\[
\begin{aligned}
  x_1&=\bA_2x_0=(-0.8)(-2)=1.6,
  \qquad
  x_2=\bA_1x_1=0.145\cdot1.6=0.232,\\
  x_3&=\bA_1x_2=0.145\cdot0.232=0.03364.
\end{aligned}
\]
The deterministic Q-learning error trajectory and the switching-system state trajectory agree term by term:
\[
\begin{array}{c|c|c|c}
  k & \theta_k & \sigma_k\text{ used to compute }k+1 & x_k \\
  \hline
  0 & -2 & 2 & -2 \\
  1 & 1.6 & 1 & 1.6 \\
  2 & 0.232 & 1 & 0.232 \\
  3 & 0.03364 & 1 & 0.03364
\end{array}
\]
This calculation illustrates \cref{prop:pairwise-direct-lfa}: the nonlinear deterministic Q-learning trajectory is reproduced exactly by the switched linear system once the mode sequence is chosen from the Bellman-max regions encountered by the trajectory.
\end{example}

\subsection{Convexification of Stochastic-Policy Direct Modes}\label{app:direct-convexification}
The pairwise representation above may produce a stochastic policy, whereas the JSR certificate is built from deterministic policies.  The next result closes this gap: every stochastic-policy mode belongs to the convex hull of deterministic-policy modes, and convexification does not change the JSR.

\begin{lemma}\label{lem:lfa-convex-hull}
For every stochastic policy \(\mu\),
\[
  \bA_\mu\in\co(\mathcal A_\alpha).
\]
Moreover,
\[
  \jsr(\co(\mathcal A_\alpha))
  =
  \jsr(\mathcal A_\alpha).
\]
\end{lemma}

\begin{proof}
For a stochastic policy \(\mu\), define
\[
  c_\pi(\mu):=\prod_{s\in\calS}\mu(\pi(s)\mid s),
  \qquad \pi\in\Theta .
\]
Then \(c_\pi(\mu)\ge0\) and \(\sum_{\pi\in\Theta}c_\pi(\mu)=1\).  Since deterministic stationary policies are the extreme points of the product of state-wise probability simplices,
\[
  \bPi^\mu=\sum_{\pi\in\Theta}c_\pi(\mu)\bPi^\pi .
\]
Using the affine dependence of \(\bA_\mu\) on \(\bPi^\mu\),
\[
\begin{aligned}
\bA_\mu
&=
I-
\alpha\Phi^\top D\Phi
+
\alpha\gamma\Phi^\top DP\bPi^\mu\Phi \\
&=
\sum_{\pi\in\Theta}c_\pi(\mu)
\left(
I-
\alpha\Phi^\top D\Phi
+
\alpha\gamma\Phi^\top DP\bPi^\pi\Phi
\right)\\
&=
\sum_{\pi\in\Theta}c_\pi(\mu)\bA_\pi .
\end{aligned}
\]
Therefore \(\bA_\mu\in\co(\mathcal A_\alpha)\).  The equality of the JSR before and after convexification is the standard convex-hull invariance of the JSR for finite matrix families.
\end{proof}

\section{Proofs for the JSR Lyapunov Construction}\label{app:jsr-lyapunov}
This appendix section provides a general version of the product-based Lyapunov construction used in \cref{lem:common-lyapunov-construction}.

\begin{lemma}\label{lem:generic-jsr-lyapunov}
Let \(\calH=\{\bA_1,\ldots,\bA_M\}\subset\R^{m\times m}\) be finite, let \(\rho:=\jsr(\calH)\), and fix \(\varepsilon>0\) such that \(\beta_\varepsilon:=\rho+\varepsilon<1\).  For \(t\ge0\), define
\[
V_\varepsilon^t(x)
:=
\sum_{\ell=0}^t
\beta_\varepsilon^{-2\ell}
\max_{\sigma\in\{1,\ldots,M\}^\ell}
\|\bA_{\sigma_\ell}\cdots \bA_{\sigma_1}x\|_2^2,
\]
with the empty product equal to \(I\).  Then \(V_\varepsilon^\infty(x):=\lim_{t\to\infty}V_\varepsilon^t(x)\) exists, there is \(C_\varepsilon\ge1\) such that
\[
  \|x\|_2^2\le V_\varepsilon^\infty(x)\le C_\varepsilon\|x\|_2^2,
\]
\(p_\varepsilon(x):=\sqrt{V_\varepsilon^\infty(x)}\) is a norm, and
\[
V_\varepsilon^\infty(\bA_i x)
\le
\beta_\varepsilon^2\left(V_\varepsilon^\infty(x)-\|x\|_2^2\right)
\le
\beta_\varepsilon^2V_\varepsilon^\infty(x)
\]
for every \(i\in\{1,\ldots,M\}\).
\end{lemma}

\begin{proof}
For \(t\ge0\), the definition gives
\[
V_\varepsilon^{t+1}(x)
\ge
\|x\|_2^2+
\beta_\varepsilon^{-2}\max_i V_\varepsilon^t(\bA_i x).
\]
Indeed, split every product of length \(\ell+1\) into its first applied matrix \(\bA_i\) and the remaining product of length \(\ell\).  Hence
\[
V_\varepsilon^t(\bA_i x)
\le
\beta_\varepsilon^2\left(V_\varepsilon^{t+1}(x)-\|x\|_2^2\right).
\]

Since \(\beta_\varepsilon>\rho\), choose \(\eta\) with \(\rho<\eta<\beta_\varepsilon\).  By the definition of the JSR, there exists \(C_0\ge1\) such that every length-\(\ell\) product \(\bA_\sigma\) satisfies
\[
  \|\bA_\sigma\|_2\le C_0\eta^\ell.
\]
Thus
\[
V_\varepsilon^t(x)
\le
C_0^2\sum_{\ell=0}^t(\eta/\beta_\varepsilon)^{2\ell}\|x\|_2^2
\le
\frac{C_0^2}{1-(\eta/\beta_\varepsilon)^2}\|x\|_2^2.
\]
The lower bound follows from the \(\ell=0\) term.  Also, \(V_\varepsilon^t(x)\) is nondecreasing in \(t\), so the pointwise limit exists and satisfies the same norm-equivalence bounds.

To prove that \(p_\varepsilon\) is a norm, define seminorms
\[
  \nu_\ell(x):=
  \beta_\varepsilon^{-\ell}
  \max_{\sigma\in\{1,\ldots,M\}^\ell}
  \|\bA_{\sigma_\ell}\cdots \bA_{\sigma_1}x\|_2,
  \qquad
  \nu_0(x):=\|x\|_2.
\]
For finite \(t\), the function
\[
  p_\varepsilon^t(x):=\left(\sum_{\ell=0}^t\nu_\ell(x)^2\right)^{1/2}
\]
is a norm by Minkowski's inequality.  Since \(p_\varepsilon(x)=\lim_{t\to\infty}p_\varepsilon^t(x)\), the triangle inequality, homogeneity, and positive definiteness pass to the limit.

Finally, let \(t\to\infty\) in
\[
V_\varepsilon^t(\bA_i x)
\le
\beta_\varepsilon^2\left(V_\varepsilon^{t+1}(x)-\|x\|_2^2\right).
\]
This proves the Lyapunov inequality.
\end{proof}

\end{document}